%% file: main.tex
\documentclass{article}

\usepackage[main, final]{neurips_2026}

\usepackage[utf8]{inputenc}
\usepackage[T1]{fontenc}
\usepackage{hyperref}
\usepackage{url}
\usepackage{booktabs}
\usepackage{amsfonts}
\usepackage{nicefrac}
\usepackage{microtype}
\usepackage{xcolor}
\usepackage{wrapfig}
\usepackage{subcaption}
\usepackage{amsmath}
\usepackage{amsthm}
\usepackage{algorithm}
\usepackage{algpseudocode}
\usepackage{tikz}
\usetikzlibrary{arrows.meta}
\usetikzlibrary{positioning}
\usetikzlibrary{calc}
\usepackage{float}
\usepackage{enumitem}
\usepackage{tabularx}
\usepackage{graphicx}

\newtheorem{theorem}{Theorem}[section]
\newtheorem{lemma}[theorem]{Lemma}
\newtheorem{proposition}[theorem]{Proposition}
\newtheorem{corollary}[theorem]{Corollary}
\newtheorem{definition}[theorem]{Definition}
\newtheorem{example}[theorem]{Example}
\newtheorem{remark}[theorem]{Remark}

\newcommand{\bs}[1]{\boldsymbol{#1}}

\DeclareFontFamily{OT1}{mathc}{}
\DeclareFontShape{OT1}{mathc}{m}{it}{<-> mathc10}{}
\DeclareMathAlphabet{\mathcal}{OT1}{mathc}{m}{it}

\title{Optimal Experiments for Partial Causal Effect Identification}

\author{%
  Tobias Maringgele \\
  Technical University of Munich \\
  \texttt{tobias.maringgele@tum.de} \\
  \And
  Jalal Etesami \\
  Technical University of Munich \\
  \texttt{j.etesami@tum.de} \\
}

\begin{document}
\maketitle

\begin{abstract}
Causal queries are often only partially identifiable from observational
data, and experiments that could tighten the resulting bounds are
typically costly. We study the problem of selecting, prior to observing
experimental outcomes, a cost-constrained subset of experiments that
maximally tightens bounds on a target query.
We formalize this as the \emph{max-potency problem}, where epistemic
potency measures the worst-case reduction in bound width guaranteed by
an experiment, and show that this problem is NP-hard via a reduction
from 0-1 knapsack. Building on the polynomial-programming framework of
\citet{duarte2023automated}, we give a general procedure for evaluating
epistemic potency in discrete settings. To control the
super-exponential search space, we introduce two graphical pruning
criteria that depend only on the causal graph and the query: a novel
path-interception rule that exploits district structure to certify zero
potency in linear time, and an identifiability check based on the ID
algorithm. On Erd\H{o}s--R\'enyi random graphs and 11 \emph{bnlearn}
benchmark networks, the two criteria together prune 50--88\% of
candidate experiments on average without solving a single polynomial
program. For the general subset search, we show that ID-pruned
experiments are combinatorially inert, yielding a super-exponential
reduction in the number of subsets evaluated. We close with an
end-to-end demonstration on observational NHANES data, selecting
optimal experiments for estimating the effect of physical activity on
diabetes.
\end{abstract}

\input{introduction}

\input{notation}

\input{bounds}

\input{potency_problem}

\input{pruning}

\input{Experiments}

\vspace{-.3cm}
\section{Conclusion}
\label{sec:conclusion}
\vspace{-.25cm}
We introduced the max-potency problem---selecting cost-constrained experiments to maximally tighten bounds on a partially identifiable causal query---defined \emph{epistemic potency} as a worst-case measure of informativeness, and showed the problem is NP-hard. Two graphical pruning criteria, a path-interception rule and an ID check, together prune 50--88\% of the super-exponential candidate space for singleton selection. For subset search, ID-pruned experiments are combinatorially inert (Corollary~\ref{thm-id_inertness}) and can be removed entirely, reducing the exponent by roughly 60\% on average. A natural extension is the sequential setting, where realized outcomes are incorporated into the base program after each round and the next optimal batch is selected against the updated bounds, as illustrated in Section~\ref{sec:evaluation}.

\vspace{-.15cm}
\textbf{Limitations.}
i) Our framework inherits limitations from \citet{duarte2023automated}: known ADMG, discrete variables with finite support, exact $P(\bs{V})$, the computational cost of the bounding programs (Appendix~\ref{app:runtime}), and possibly non-sharp bounds if the nonconvex polynomial program is terminated early.
ii) Our pruning criteria are sound but not complete: experiments not certified useless may still have zero potency, and interception-pruned experiments may contribute non-zero potency in subsets.
iii) Our worst-case bound-width definition (Section~\ref{sec:bounds}) guarantees tightening regardless of outcome; under alternative criteria (Appendix~\ref{sec:alternative_W_cal}) the optimal set may differ.

\begin{ack}
We thank the anonymous reviewers and area chairs for their constructive feedback.
This work received no third-party funding. The authors declare no competing interests.
\end{ack}

\bibliographystyle{abbrvnat}
\bibliography{references}

\appendix

\input{appendix}


\end{document}

%% file: introduction.tex
\vspace{-0.1cm}
\section{Introduction}\label{sec:intro}
\vspace{-0.2cm}
A central goal of causal inference is to estimate the effect of a treatment on an outcome. Population-level summaries such as the average treatment effect (ATE) capture aggregate behavior, but many decisions hinge on finer-grained counterfactual quantities---for example, the probability that a treatment is both necessary and sufficient to produce an outcome (PNS) \citep{Pearl1999, mueller_pearl_2022}. Such queries are rarely point-identifiable from observational data; the best one can do is derive \emph{bounds} \citep{balke1997bounds, tian_probabilities_2000}, which are often too wide to be decision-relevant. Tightening them requires additional data, and additional data means experiments, which are costly and subject to budget constraints.

This motivates the central question of this paper: \emph{given a causal graph, an observational distribution, and a budget, which subset of experiments, chosen ex ante, most tightens bounds on a target query?} Answering it well means reasoning about how informative an experiment will be \emph{before} running it, unlike sequential designs that adapt to realized outcomes \citep{ailer2024targeted}, and doing so over a search space that grows super-exponentially in the number of variables. 

Our contributions are as follows:
\begin{itemize}[leftmargin=8pt]
    \item We formalize this as the \emph{max-potency problem}, where
    epistemic potency measures the worst-case reduction in bound width
    an experiment guarantees, and show it is NP-hard via a reduction
    from 0-1 knapsack (Section~\ref{sec:max_potency}).
    \item We give a general procedure for evaluating epistemic potency
    in discrete settings, building on the polynomial-programming
    framework of \citet{duarte2023automated}
    (Section~\ref{sec:formulating}).
    \item We introduce two graphical pruning criteria that depend only
    on the causal graph and the query to certify zero potency of
    candidate experiments without solving any polynomial program: a
    novel path-interception rule that runs in linear time by exploiting
    the graph's district structure, and an identifiability check based
    on the ID algorithm \citep{shpitser2006identification}. We further
    show that ID-pruned experiments are \emph{combinatorially inert}:
    they can be removed from the subset search entirely
    (Section~\ref{sec:pruning}).
    \item On Erd\H{o}s--R\'enyi random graphs and 11 \emph{bnlearn}
    benchmark networks, the two criteria together prune 50--88\% of
    candidate experiments on average when selecting a single optimal experiment.
    For the general subset search, ID-inertness alone reduces the
    search space exponent by roughly $60\%$ on average---a
    super-exponential reduction in the number of subsets that need to be evaluated. We
    demonstrate the full pipeline end-to-end on observational NHANES
    data (Section~\ref{sec:evaluation}).
\end{itemize}

\vspace{-.3cm}
\subsection{Related Work}\label{sec:related} 
\vspace{-.1cm}
Our work bridges two largely separate threads in causal inference: causal
\emph{identification}, which asks whether a query can be uniquely computed from
the available data, and \emph{partial identification}, which asks how tightly
the query can be bounded when unique computation is impossible. Optimal
experimental design has been extensively studied for the former; we address it
for the latter.

\textbf{Causal identification.}
\citet{shpitser2006identification} present a sound and complete algorithm (the \emph{ID algorithm}) for determining whether the causal effect of a subset of variables, $\bs{X}$, on a target subset, $\bs{Y}$, denoted by $P(\bs{y} \mid do(\bs{x}))$, is identifiable from observational data alone, which holds if and only if a specific graphical structure known as a \emph{hedge} is absent.
\citet{bareinboim2012causal} extend this to settings where a specific form of experimental data is also available, introducing
\emph{z-identifiability}: a query $P(\bs{y} \mid do(\bs{x}))$ is
\emph{z-identifiable} if it can be uniquely computed from the observational distribution $P(\bs{V})$ together with distributions obtained by intervening on a designated subset of variables $\bs{Z}$. Intuitively, interventions on $\bs{Z}$ can break certain confounding structures that prevent identification from observations alone. 
\citet{lee2019general} further extend this to \emph{general identifiability} (g-identifiability), where the available data is an arbitrary collection of interventional distributions of the form $\{P(\bs{V}\setminus \bs{Z}_i\mid do(\bs{Z}_i))\}_{i=0}^m$ rather than a single family.
\citet{kivva2022revisiting} provide a sound and complete algorithm for g-identifiability, reducing it to a sequence of classical ID sub-problems.

\textbf{Experimental design.}
\citet{AkbariEtesamiKiyavash2025} formulate the minimum-cost intervention design (MCID) problem for \emph{full} identification, show that it is NP-hard, and propose an algorithm based on minimum hitting sets. \citet{elahi2024fast} reformulate MCID as weighted partial MAX-SAT and integer linear programming instances, achieving orders-of-magnitude speedups in practice. Both focus on point identification. Closer to our regime, \citet{ailer2024targeted} sequentially design indirect experiments to narrow bound gaps in nonlinear, confounded IV settings; we instead study ex-ante subset selection in discrete graphs, with graphical certificates of zero potency.

\textbf{Partial identification and causal bounds.}
\citet{balke1997bounds} derive sharp (i.e.\ tightest possible) bounds on binary causal effects using linear programming over response-type variables; their approach is restricted to discrete, in particular binary, settings with a specific graph structure. \citet{tian_probabilities_2000} provide closed-form bounds for the probabilities of causation (PNS, PN, PS) for a single binary treatment and outcome, also requiring discrete variables. \citet{zhang2022partial} bound counterfactual queries from arbitrary combinations of observational and interventional distributions via canonical SCMs. \citet{duarte2023automated} further generalize to arbitrary causal and counterfactual queries in discrete settings via polynomial programming. \citet{sachs_general_2023} characterize conditions under which these polynomial programs reduce to linear programs. \citet{shridharan2023scalable} extend this class with response-type aggregation to prune the underlying LP. \citet{shridharan2023quasimarkovian} further show that on quasi-Markovian graphs \citep{zaffalon2020structural}, the polynomial degree reduces to the number of \emph{intervened} c-components, yielding LP or QP formulations when at most two c-components are intervened upon. Like ours, the latter exploits a c-factor decomposition \citep{tianPearlQfactor}, but to reduce the structural complexity of the inner program on a restricted graph class; we instead use c-factor structure (Section~\ref{sec:pruning-relevant_res}) to certify experiments as uninformative ex ante on general ADMGs. In a related vein, \citet{li_epsilon_2026} characterize when sharp bounds have width at most $2\epsilon$ for binary probabilities of causation.

%% file: notation.tex
\vspace{-.3cm}
\section{Preliminaries}\label{sec:preliminaries}
\vspace{-.2cm}
\textbf{Notation.} Random variables are denoted by capital letters and their values by lowercase letters; bold denotes sets. For a random variable $X$, we write $supp(X)$ to refer to its support. When $X$ is a binary variable, i.e. $|supp(X)|=2$, we use $x$ for the case $X=1$ (or $\texttt{true}$) and $x'$ for $X=0$ (or $\texttt{false}$).

\textbf{Structural Causal Model (SCM).} An SCM is a tuple $\mathcal{M} = (\bs{V}, \bs{U}, \bs{F}, P(\bs{U}))$ of observed (endogenous) variables $\bs{V}$, latent (exogenous) variables $\bs{U}$, structural equations $\bs{F}=\{f_V\}_{V \in \bs{V}}$, where each $f_V$ determines the value of $V$ as a function of its (latent and observed) parents.

\textbf{Acyclic Directed Mixed Graph (ADMG).} The causal structure of an SCM is summarized by an ADMG $\mathcal{G} = (\bs{V}, E^d, E^b)$ over observed variables, with directed edges $V_i \to V_j \in E^d$ indicating that $V_j$ structurally depends on $V_i$, and bidirected edges $V_i \leftrightarrow V_j \in E^b$ indicating a hidden common cause of $V_i$ and $V_j$. The directed part is acyclic. For a node $V \in \bs{V}$, $\bs{pa}(V)$, $\bs{ch}(V)$, and $\bs{anc}(V)$ denote its parents, children, and ancestors, respectively; these extend to sets naturally. In the figures of this paper, hidden common causes are shown as dashed nodes with directed edges to their children; this is equivalent to bidirected edges in the ADMG representation.

An \emph{intervention} on $\bs{X} \subseteq \bs{V}$, denoted by $do(\bs{X} = \bs{x})$, replaces each structural equation $f_X$ for $X \in \bs{X}$ with the constant value $x \in supp(X)$, removing all incoming edges to $\bs{X}$ in $\mathcal{G}$. The resulting \emph{post-intervention distribution} on $\bs{Y} \subseteq \bs{V} \setminus \bs{X}$, $ P(\bs{Y}_{\bs{x}})=P(\bs{Y} \mid do(\bs{X} = \bs{x}))$, may differ from the observational conditional $P(\bs{Y} \mid \bs{X} = \bs{x})$ whenever hidden confounders are present. 
We write $P(\bs{Y}|do(\bs{X}))$ to denote the family of post-intervention distributions indexed by $\bs{x}\in supp(\bs{X})$.
A \emph{counterfactual} is a statement that involves contradictory interventions. For example, in the binary case, a counterfactual probability may take the form $P(y_x, y'_{x'})$ for variables $X, Y$. 
In this paper, we will only consider counterfactuals that are conjunctive statements like the above, where each conjunct is a counterfactual \emph{world}. We denote the set of all counterfactual worlds appearing in a given statement by $\mathfrak{C}$. In the example above, $|\frak{C}|=2$, with interventions on $X^{(\frak{c})}$ and outcomes on $Y^{(\frak{c})}$ for each world $\frak{c}\in \frak{C}$.

\textbf{Discrete variables and response-type variables.} We assume all variables in $\bs{V}$ are discrete with known finite supports. Since each structural equation $f_V$ maps finitely many parent configurations to a finite output, it can only encode a finite number of distinct input-output behaviors. For each latent disturbance $U_k$, these behaviors over the children of $U_k$ are indexed by a discrete \emph{response-type} $R_k$. The collection $\bs{R} = \{R_k\}_{U_k \in \bs{U}}$ is finite, and a fixed $\bs{r} \in supp(\bs{R})$ fully determines the model (including all counterfactuals). Replacing latent variables $\bs{U}$ with $\bs{R}$ yields an equivalent model with no loss of generality \citep[Proposition~1]{duarte2023automated}. Given $(\mathcal{G}, P(\bs{V}), supp(\bs{V}))$ and a target query, the polynomial-programming framework of \citet{duarte2023automated} constructs a program over $\bs{R}$ whose optima are sharp bounds on the query. See Appendices~\ref{app:response-type-def} (formal construction), \ref{app:primer} (primer), and \ref{app:notation} (notation table).

%% file: bounds.tex
\vspace{-.1cm}
\section{Causal Bounds and Epistemic Potency}\label{sec:bounds}
\vspace{-.1cm}
\textbf{Setup.} We are given an ADMG $\mathcal{G}$ over observed variables $\bs{V}$ (all discrete) with distribution $P(\bs{V})$ and a causal query $\theta$, which may correspond to a post-intervention, a counterfactual, or a functional thereof. Note that using the framework of \citet{duarte2023automated}, the sharp observational bounds $L_\theta \leq \theta \leq U_\theta$ are obtained by optimizing a polynomial program over $\bs{R}$.

We define $\mathcal{A}:=\{(y_x),(y_x,y'_{x'}),\dots\}$ as a set of \emph{candidate experiments}, i.e., arbitrary polynomial functions of the response-type variables $\bs{R}$. We use the term broadly: $\mathcal{A}$ may contain physically realizable experiments such as randomized controlled trials and  interventions \citep{raghavan2025counterfactual}, as well as structural assumptions defensible in the application domain, such as monotonicity. The polynomial program treats both uniformly as constraints on $\bs{R}$.
For any $\mathcal{a} \subseteq \mathcal{A}$ with outcomes $p_{\mathcal{a}}$, adding the constraints $f_{\mathcal{a}}(\bs{R}) = p_{\mathcal{a}}$ to the polynomial program yields new sharp bounds, denoted by $U_\theta(\mathcal{a}, p_\mathcal{a})$ (upper) and $L_\theta(\mathcal{a}, p_\mathcal{a})$ (lower). We write $U_\theta$ and $L_\theta$ for the bounds without any experimental constraint (i.e., $\mathcal{a} = \emptyset$). We then define the \textit{worst-case bound width} as follows:
$$
\mathcal{W}_\theta(\mathcal{a}) := \underset{\substack{p_{\mathcal{a}}\in\prod_{\tilde{\mathcal{a}} \in \mathcal{a}}[l_{p_{\tilde{\mathcal{a}}}}, u_{p_{\tilde{\mathcal{a}}}}]}}{\max} \bigl(U_\theta(\mathcal{a}, p_{\mathcal{a}}) - L_\theta(\mathcal{a}, p_{\mathcal{a}})\bigr),
$$
where $l_{p_{\tilde{\mathcal{a}}}}$ and $u_{p_{\tilde{\mathcal{a}}}}$ are the sharp observational bounds on $p_{\tilde{\mathcal{a}}}$, itself a causal query. 
Intuitively, $\mathcal{W}_\theta(\mathcal{a})$ asks: if we commit to $\mathcal{a}$ now, how wide could the bounds still be once its outcome arrives---a worst-case width that we can promise in advance.
We take the maximum because bounds must hold regardless of the experimental outcome; any smaller choice could yield a bound width that is exceeded by the realized outcome.\footnote{Possible alternatives are briefly discussed in Appendix~\ref{sec:alternative_W_cal}.} We denote the observational bound width as $\mathcal{W}_\theta:=\mathcal{W}_\theta(\emptyset)=U_\theta-L_\theta$.

For illustration, Figure~\ref{fig:complexDAG} shows an ADMG with 
$\bs{V} = \{A, B, X, M, Y, C, D\}$ (all binary). Restricting $\mathcal{A}$ 
to single-world interventions $P(\bs{W} \mid do(\bs{Z}{=}\bs{z}))$ with 
$\bs{W} \cap \bs{Z} = \emptyset$, each $V \in \bs{V}$ falls into exactly 
one of three roles: intervened with a chosen value, 
measured as an outcome, or neither. This results in
$
|\mathcal{A}| \;=\; \prod_{V \in \bs{V}} \big(|supp(V)| + 2\big) 
\;=\; 4^{7} \;=\; 16{,}384
$
candidate experiments.

\begin{figure}[t]
\centering
\begin{subfigure}{0.34\textwidth}
\centering
\scalebox{0.78}{
\begin{tikzpicture}[->, node distance=1.5cm, thick]
  \node[draw, circle] (X) {\( X \)};
  \node[draw, circle, right of=X] (Y) {\( Y \)};
  \node[draw, circle, dashed, above of=X] (U) {\( U \)};
  \draw (X) -- (Y);
  \draw[dashed] (U) -- (X);
  \draw[dashed] (U) -- (Y);
\end{tikzpicture}
}
\caption{An ADMG with two observed variables and a hidden confounder.}
\label{fig:confDAG}
\end{subfigure}
\hfill
\begin{subfigure}{0.51\textwidth}
\centering
\scalebox{0.55}{
\begin{tikzpicture}[->, node distance=1.6cm, thick]
  \tikzstyle{nodeStyle}=[draw, circle, minimum size=8mm]
  \node[nodeStyle] (A) {\( A \)};
  \node[nodeStyle, right of=A] (B) {\( B \)};
  \node[nodeStyle, right of=B] (X) {\( X \)};
  \node[nodeStyle, right of=X] (M) {\( M \)};
  \node[nodeStyle, right of=M] (Y) {\( Y \)};
  \node[nodeStyle, below of=X, yshift=-0.3cm] (C) {\( C \)};
  \node[nodeStyle, right of=C] (D) {\( D \)};

  \node[nodeStyle, dashed, above of=B, yshift=0.3cm] (R3) {\( R_3 \)};
  \node[nodeStyle, dashed, above of=X, xshift=0.9cm, yshift=0.3cm] (R1) {\( R_1 \)};
  \node[nodeStyle, dashed, above of=M, xshift=0.9cm, yshift=0.3cm] (R2) {\( R_2 \)};
  \node[nodeStyle, dashed, below of=B, yshift=-0.3cm] (R4) {\( R_4 \)};

  \draw (A)--(B); \draw (B)--(X); \draw (X)--(M); \draw (M)--(Y);
  \draw (X)--(C); \draw (C)--(D); \draw (D)--(M);

  \draw[dashed] (R3)--(A); \draw[dashed] (R3)--(B);
  \draw[dashed] (R1)--(X); \draw[dashed] (R1)--(M);
  \draw[dashed] (R2)--(M); \draw[dashed] (R2)--(Y);
  \draw[dashed] (R4)--(C); \draw[dashed] (R4)--(B);
\end{tikzpicture}
}
\caption{An ADMG with seven observed binary variables and four hidden common causes.}
\label{fig:complexDAG}
\end{subfigure}
\caption{Two example ADMGs.}\label{fig:admgs}
\vspace{-.4cm}
\end{figure}
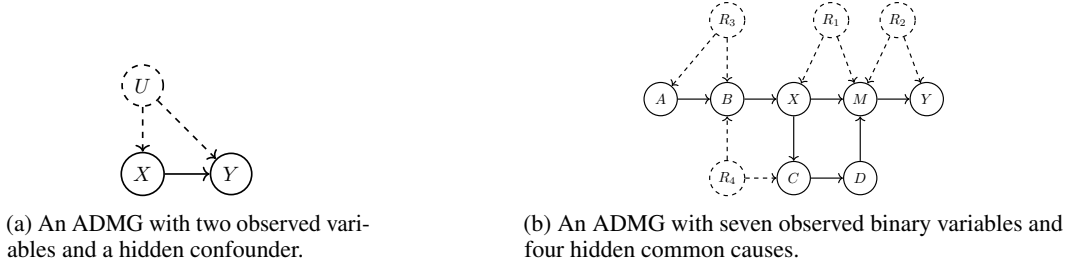

\begin{example}[Observational bounds on PNS]\label{ex:obs_bounds}
Consider a binary treatment $X$ with causal effect on a binary outcome $Y$, confounded by a hidden variable $U$. The ADMG is shown in Figure~\ref{fig:confDAG}. Our query is the probability of necessity and sufficiency $\text{PNS} = P(y_x, y'_{x'})$, where $y_x$ denotes $Y{=}y \mid do(X{=}x)$.

\citet{tian_probabilities_2000} showed that the sharp bounds for the PNS are:
\begin{small}
\[
\max \left\{\!\!\!
\begin{array}{l}
0,\  P(y_x) - P(y_{x'}), \\
P(y) - P(y_{x'}), P(y_x) - P(y)
\end{array}
\!\!\!\right\}
\!\leq\! \text{PNS} \!\leq \!
\min \!\left\{\!\!\!
\begin{array}{l}
P(y_x),\ P(x,y) + P(x',y'), \\
P(y'_{x'}),\  P(y_x) - P(y_{x'}) + P(x,y') + P(x',y)
\end{array}
\!\!\!\right\}.
\]
\end{small}
In this example, two candidate experiments are available: $\tilde{\mathcal{a}}_1: y_x$ and $\tilde{\mathcal{a}}_2: y_{x'}$. 
Since $P(y_x)$ and $P(y_{x'})$ are not identifiable from observational data, only the purely observational terms in the above bounds remain, giving $0 \leq \text{PNS} \leq P(x,y)+P(x',y')$. 
For instance, by assuming $P(x,y)=0.2,   P(x,y')=0.1,  P(x',y)=0.3,$  and $P(x',y')=0.4$, we get $\mathcal{W}_\theta(\emptyset) = 0.6$.
\end{example}
\begin{definition}[Epistemic Potency]
\label{def:potency}
For any $\mathcal{a} \subseteq \mathcal{A}$ and a given causal query $\theta$, we define 
$pot_\theta(\mathcal{a}) := \mathcal{W}_\theta - \mathcal{W}_\theta(\mathcal{a})$. If $pot_\theta(\mathcal{a})=0$, we call $\mathcal{a}$ \emph{useless}.
For a single experiment $\tilde{\mathcal{a}} \in \mathcal{A}$, we write $pot_\theta(\tilde{\mathcal{a}})$ and $\mathcal{W}_\theta(\tilde{\mathcal{a}})$ as shorthand for $pot_\theta(\{\tilde{\mathcal{a}}\})$ and $\mathcal{W}_\theta(\{\tilde{\mathcal{a}}\})$, respectively.
\end{definition}
Potency measures how much the worst-case bound width on a causal query $\theta$ is reduced by performing experiment $\mathcal{a}$. Since $\mathcal{W}_\theta(\mathcal{a})$ is a width we can promise before seeing any outcome, $pot_\theta(\mathcal{a})$ is the tightening we can promise in advance. The following result establishes that potency is always non-negative. In words, experiments can only help, never hurt.
\begin{lemma}
\label{lemma-pot_greater_zero}
For any $\mathcal{a} \subseteq \mathcal{A}$,  $pot_\theta(\mathcal{a}) \geq 0$.
\end{lemma}
\vspace{-.2cm}
This is because adding experimental constraints restricts the set of feasible response-type configurations, which can only tighten the resulting  bounds. All detailed proofs are in Appendix~\ref{sec:proofs}. 
The next result gives a useful characterization of zero potency that will underpin our pruning criteria: an experiment has zero potency if and only if some feasible outcome leaves both bounds unchanged.
\begin{lemma}\label{lemma_pot0_implies_boundsEqual}
For any $\mathcal{a} \subseteq \mathcal{A}$, $pot_\theta(\mathcal{a}) = 0$ if and only if $\exists\, p_{\mathcal{a}}^{\dagger} \in \prod_{\tilde{\mathcal{a}} \in \mathcal{a}} [l_{p_{\tilde{\mathcal{a}}}}, u_{p_{\tilde{\mathcal{a}}}}]$ s.t.\ $U_\theta(\mathcal{a}, p_{\mathcal{a}}^{\dagger}) = U_\theta$ and $L_\theta(\mathcal{a}, p_{\mathcal{a}}^{\dagger}) = L_\theta$.
\end{lemma}

\begin{example}
\label{ex:potency}
Consider the setting of Example \ref{ex:obs_bounds}. 
(1) Performing $\tilde{\mathcal{a}}_1$: From \citep{tian_probabilities_2000}, we obtain that $p_{\tilde{\mathcal{a}}_1} \in [P(x,y),\, 1-P(x,y')] = [0.2, 0.9]$. (2) Performing $\tilde{\mathcal{a}}_2$: Similarly, we have $p_{\tilde{\mathcal{a}}_2} \in [0.3, 0.6]$. For each, we obtain the following bounds:
\begin{small}
\begin{align*}
U_\theta(\!\tilde{\mathcal{a}}_1, p_{\tilde{\mathcal{a}}_1}\!)\!\! -\!\! L_\theta(\!\tilde{\mathcal{a}}_1, p_{\tilde{\mathcal{a}}_1}\!)
\!=\! \begin{cases}
\!p_{\tilde{\mathcal{a}}_1} &\!\!\!  p_{\tilde{\mathcal{a}}_1}\! \in \![0.2, 0.5], \\
\!0.5 &\!\!\!  p_{\tilde{\mathcal{a}}_1} \!\in\! (0.5, 0.6], \\
\!1.1\! -\! p_{\tilde{\mathcal{a}}_1} &\!\!\!  p_{\tilde{\mathcal{a}}_1} \!\in\! (0.6, 0.9].
\end{cases}
\ 
U_\theta(\!\tilde{\mathcal{a}}_2, p_{\tilde{\mathcal{a}}_2}\!)\!\! -\!\! L_\theta(\!\tilde{\mathcal{a}}_2, p_{\tilde{\mathcal{a}}_2}\!)\! =\!
\begin{cases}
\!0.1\! +\! p_{\tilde{\mathcal{a}}_2} &\!\!\! p_{\tilde{\mathcal{a}}_2} \!\in \![0.3, 0.4], \\
\!0.5 &\!\!\! p_{\tilde{\mathcal{a}}_2}\! \in\! (0.4, 0.5], \\
\!1 \!-\! p_{\tilde{\mathcal{a}}_2} &\!\!\! p_{\tilde{\mathcal{a}}_2}\! \in \![0.5, 0.6].
\end{cases}
\end{align*}
\end{small}
This implies the worst-case widths are $\mathcal{W}_\theta(\tilde{\mathcal{a}}_1) = 0.5$ and $\mathcal{W}_\theta(\tilde{\mathcal{a}}_2)\! =\! 0.5$.

(3) Performing both: Let $\mathcal{a}_3\! =\! \{\tilde{\mathcal{a}}_1\!,  \tilde{\mathcal{a}}_2\}$.
Maximizing $U_\theta(\mathcal{a}_3, p_{\tilde{\mathcal{a}}_1},p_{\tilde{\mathcal{a}}_2})\! -\! L_\theta(\mathcal{a}_3, p_{\tilde{\mathcal{a}}_1},p_{\tilde{\mathcal{a}}_2})$ over $(p_{\tilde{\mathcal{a}}_1},p_{\tilde{\mathcal{a}}_2})\! \in\! [0.2,0.9]\!\times\![0.3,0.6]$ yields $\mathcal{W}_\theta(\mathcal{a}_3) = 0.4$. Therefore, $pot_\theta(\tilde{\mathcal{a}}_1) = 0.1$, $pot_\theta(\tilde{\mathcal{a}}_2) = 0.1$, $pot_\theta(\mathcal{a}_3) = 0.2$. Thus, $\mathcal{a}_3$ strictly outperforms either individual experiment, illustrating that optimal design requires reasoning over \emph{sets} of candidates. See Figure~\ref{fig:a_W} in Appendix~\ref{app:example}.
\end{example}


%% file: potency_problem.tex
\vspace{-.2cm}
\section{The Max-Potency Problem}\label{sec:max_potency}
\vspace{-.2cm}
Here, we introduce our problem. Consider a cost function $c: 2^\mathcal{A} \to [0, \infty]$ that assigns a nonnegative cost to any combination of experiments in $\mathcal{A}$, along with a budget $B$. This function encodes the analyst's burden uniformly across both physical experiments and structural assumptions in $\mathcal{A}$---running an experiment or committing to an assumption. We allow infinite cost for impossible experiments, and assume at least one $\mathcal{a} \subseteq \mathcal{A}$ has cost at most $B$. The \emph{max-potency problem} is given by
\[
\underset{\mathcal{a}\subseteq\mathcal{A}}{\text{max}} \ pot_\theta(\mathcal{a}) \quad \text{s.t.} \quad  c(\mathcal{a}) \leq B. 
\]

\begin{theorem}
\label{thm:np-hard}
The max-potency problem is NP-hard.
\end{theorem}

 We reduce 0-1 knapsack to max-potency by constructing, for each knapsack item, a disjoint confounded treatment-outcome cell whose per-cell potency can be tuned to any target value in a non-degenerate range. Disjointness makes the districts independent, and the resulting potency is additive across cells, so the max-potency optimum coincides with the knapsack optimum. See Appendix~\ref{app:np-hardness}.

\begin{example}
\label{ex:max_potency}
Consider the ADMG in Figure \ref{fig:confDAG}, with the total budget $B=1$ and each of $\tilde{\mathcal{a}}_1, \tilde{\mathcal{a}}_2$ costs $1$; the combined experiment $\mathcal{a}_3 = \{\tilde{\mathcal{a}}_1, \tilde{\mathcal{a}}_2 \}$ costs $2$. The next table summarizes the resulting potencies. 
\begin{center}
\begin{tabular}{lccc}
\toprule
Experiment set $\mathcal{a}$ & $\mathcal{W}_\theta(\mathcal{a})$ & $pot_\theta(\mathcal{a})$ & Cost \\
\midrule
$\emptyset, \tilde{\mathcal{a}}_1, \tilde{\mathcal{a}}_2, \mathcal{a}_3$ & $0.6, 0.5, 0.5, 0.4$ & $0, 0.1, 0.1, 0.2$ & $0, 1, 1, 2$ \\
\bottomrule
\end{tabular}
\end{center}
While $\mathcal{a}_3$ achieves the largest potency, it is infeasible under $B=1$. Both $\tilde{\mathcal{a}}_1$ and $\tilde{\mathcal{a}}_2$ are optimal, each guaranteeing a reduction of $0.1$.
\end{example}

For completeness, Appendix~\ref{subsec:solving-max-potency} gives an exact enumeration algorithm (Algorithm~\ref{alg:solve-max-potency}) that skips useless candidates identified by the pruning criteria of Section~\ref{sec:pruning}.


\vspace{-.1cm}
\section{Constructing the Polynomial Program}
\label{sec:formulating}
\vspace{-.1cm}
We compute the bounds $L_\theta$ and $U_\theta$ via the polynomial-programming framework of \citet{duarte2023automated}; we summarize the construction here and defer algorithmic details to Appendix~\ref{app:program}.

\textbf{Base program.} Given a canonical ADMG $\mathcal{G}$, an
observational distribution $P(\bs{V})$, supports $supp(\bs{V})$, and a
target query $\theta$, \citet{duarte2023automated} produce a polynomial
program $\mathcal{P}_\emptyset = (\mathcal{T}, \mathcal{C}_\emptyset)$
whose decision variables are the response-type probabilities
$\{P(\bs{r}) : \bs{r} \in supp(\bs{R})\}$. 
\emph{Target polynomial} $\mathcal{T}$ expresses the causal query $\theta$ in the decision variables. That is,
$\mathcal{T} = \sum_{\bs{r} \in h(\theta)} P(\bs{r})$,
where $h(\theta) \subseteq supp(\bs{R})$ collects the configurations
under which $\theta$ holds (Definition~\ref{def:compatible-restypes}).  
\emph{Constraint set} $\mathcal{C}_\emptyset$ comprises the probability axioms $P(\bs{r}) \geq 0$ and $\sum_{\bs{r}} P(\bs{r}) = 1$, 
together with the observational constraint, i.e., $\sum_{\bs{r} \in h(\bs{v})} P(\bs{r}) = P(\bs{v})$ for each $\bs{v} \in supp(\bs{V})$. 
Minimizing (maximizing) $\mathcal{T}$ subject to $\mathcal{C}_\emptyset$ 
yields the sharp observational lower (upper) bound on $\theta$.

\textbf{Adding experiments.} For $\mathcal{a}\subseteq \mathcal{A}$, each single experiment $\tilde{\mathcal{a}} \in 
\mathcal{a}$ with outcome $p_{\tilde{\mathcal{a}}}$ contributes 
a polynomial constraint of the form:
$
f_{\tilde{\mathcal{a}}}(\bs{R}) := 
\sum_{\bs{r} \in h(\tilde{\mathcal{a}})} P(\bs{r}) 
= p_{\tilde{\mathcal{a}}}.
$
That is, $f_{\tilde{\mathcal{a}}}$ aggregates the probability mass 
across configurations compatible with $\tilde{\mathcal{a}}$, which the 
data fixes to $p_{\tilde{\mathcal{a}}}$.
We write $p_{\mathcal{a}} = (p_{\tilde{\mathcal{a}}})_{\tilde{\mathcal{a}} \in \mathcal{a}}$ for the collection of outcomes.
Setting
$\mathcal{C}_{\mathcal{a}} := \mathcal{C}_\emptyset \cup 
\big\{f_{\tilde{\mathcal{a}}}(\bs{R}) = p_{\tilde{\mathcal{a}}} : 
\tilde{\mathcal{a}} \in \mathcal{a}\big\}$ gives the augmented program 
$\mathcal{P}_{\mathcal{a}} = (\mathcal{T}, \mathcal{C}_{\mathcal{a}})$, 
whose optima are denoted by $L_\theta(\mathcal{a}, p_{\mathcal{a}})$ and 
$U_\theta(\mathcal{a}, p_{\mathcal{a}})$.
The quantity $\mathcal{W}_\theta(\mathcal{a})$ is therefore obtained by
maximizing $U_\theta(\mathcal{a}, p_{\mathcal{a}}) - L_\theta(\mathcal{a}, p_{\mathcal{a}})$ over feasible values of
$p_{\mathcal{a}}$. 
See Appendix~\ref{subsec:evaluating-potency} for a coupled-program formulation, which computes this worst-case width without explicitly enumerating $p_{\mathcal{a}}$. 

%% file: pruning.tex
\vspace{-.1cm}
\section{Pruning Useless Experiments}\label{sec:pruning}
\vspace{-.1cm}
The max-potency problem is computationally challenging for three reasons: (1) finding the optimal subset is NP-hard even if $pot_\theta$ were easy to evaluate; (2) evaluating $pot_\theta$ requires solving polynomial programs; (3) $|\mathcal{A}|$ grows exponentially in $|\bs{V}|$. We cannot easily circumvent (1) or (2), but (3) offers room: we can identify $\mathcal{A}^\dagger \subseteq \mathcal{A}$ such that $\forall \tilde{\mathcal{a}}^\dagger \in \mathcal{A}^\dagger: pot_\theta(\tilde{\mathcal{a}}^\dagger) = 0$, a set of useless experiments, \emph{without} solving any polynomial program.
\vspace{-.2cm}
\subsection{Relevant Response-Types}
\label{sec:pruning-relevant_res}
\vspace{-.2cm}
To find and filter out the useless experiments, we exploit a result of 
\citet{duarte2023automated} that shows if every directed path from a 
response-type to an outcome variable of $\tilde{\mathcal{a}}$ passes through an 
intervened variable, that response-type is not needed for expressing 
$\tilde{\mathcal{a}}$.
We extend this result and show that the base program $\mathcal{P}_\emptyset$ introduced in Section~\ref{sec:formulating} optimizes only over a subset of response-type variables $\bs{R}_\theta^* \subseteq \bs{R}$. Consequently, any experiment whose relevant response-types lie entirely outside $\bs{R}_\theta^*$ must have zero potency.

\textbf{Strategy overview.} We proceed in four steps: (1) identify \emph{query-relevant} response-types $\bs{R}_\theta$ (those with an unblocked directed path to a query outcome); (2) expand via the \emph{district hull} to $\bs{R}_\theta^*$ (all response-types in any district intersecting $\bs{R}_\theta$); (3) show via a Cartesian product decomposition and district factorization that $\mathcal{P}_\emptyset$ depends only on $\bs{R}_\theta^*$; (4) show that any experiment $\tilde{\mathcal{a}}$ with $R^\star(\tilde{\mathcal{a}}) \cap \bs{R}_\theta^* = \emptyset$ has zero potency, where $R^\star(\tilde{\mathcal{a}})$ denotes the response-type variables relevant for $\tilde{\mathcal{a}}$.

\begin{definition}[Compatible Response-Types]
\label{def:compatible-restypes}
Let $\tilde{\mathcal{a}}$ be a joint statement of counterfactual worlds $\frak{C}$, each $\frak{c} \in \frak{C}$ specifying interventions $\bs{X}^{(\frak{c})}$ and outcomes $\bs{Y}^{(\frak{c})}$. The compatible configurations are:
$$
h(\tilde{\mathcal{a}}):= \big\{ \bs{r} \in supp(\bs{R}) : \forall \frak{c} \in \frak{C},\; \forall\, V \in \bs{Y}^{(\frak{c})},\; F^{(\frak{c})}_V(\bs{r}) = v \big\},
$$
where $F^{(\mathfrak{c})}_V(\bs{r})$ recursively evaluates the value of $V$ in world $\mathfrak{c}$, using the intervention value if $V \in \bs{X}^{(\mathfrak{c})}$ and otherwise applying the structural equation.
\end{definition}
In words, $h(\tilde{\mathcal{a}})$ is the set of response-type configurations under which $\tilde{\mathcal{a}}$ holds, so the probability of $\tilde{\mathcal{a}}$ is the total mass on that set: $P(\tilde{\mathcal{a}}) = \sum_{r \in h(\tilde{\mathcal{a}})} P(r)$. This sum is how both queries and experiments enter the polynomial program (Section~\ref{sec:formulating}), but in different roles. For the query $\theta$, it is the objective $\mathcal{T}$ that we minimize and maximize. For an experiment $\tilde{\mathcal{a}}$, it is set equal to the experiment's outcome, which adds the constraint $\sum_{r \in h(\tilde{\mathcal{a}})} P(r) = p_{\tilde{\mathcal{a}}}$. Because experiments are chosen before they are run, $p_{\tilde{\mathcal{a}}}$ is not yet known and is treated as a parameter ranging over its feasible values.

\begin{example}
\label{ex:hPNS}
In Figure~\ref{fig:confDAG}, $X$ and $Y$ are binary. Let the tuple $\!(R_X, R_Y)$ be our response-type replacing $U$. $R_X\! \in \!\{x', x\}$ encodes the constant
value taken by $X$. $R_Y$ encodes the function $\{x', x\}\! \to\! \{y', y\}$ as the pair $\big(f_Y(x'),\, f_Y(x)\big)\!\in\!\{(y',y'),\,(y',y),\,(y,y'),\,(y,y)\}$. The $\tilde{\mathcal{a}}=\text{PNS}\! =\! P(y_x, y'_{x'})$
requires $f_Y(x, r_Y)\! =\! y$ and $f_Y(x', r_Y)\! =\! y'$, which uniquely
selects $r_Y\! =\! (y', y)$ (the \emph{complier}); $r_X$ is unconstrained,
thus 
$
h(\text{PNS})\! =\! \big\{(x',\,(y',y)),\;(x,\,(y',y))\big\}.
$ 
\end{example}

\begin{definition}[Districts $\bs{D}$, District-Compatible Response-Types]
\label{def:district_compatible_restype}

The \emph{districts} $\bs{D}=\{\bs{D}_1,\dots,\bs{D}_c\}$ are the connected components of the bidirected graph obtained by replacing each $R \in \bs{R}$
with a bidirected edge between its children. $\bs{D}$ partitions $\bs{V} \cup \bs{R}$. 
For any realization $\bs{v} \in supp(\bs{V})$ and subset of districts $\bs{D}' \subseteq \bs{D}$ with observed variables $\bs{V}_{\bs{D}'}=\bs{D}'\cap\bs{V}$ and response-types $\bs{R}_{\bs{D}'}=\bs{D}'\cap\bs{R}$, let
 $$
 h_{\bs{D}'}(\bs{v}) := \big\{ \bs{r}_{\bs{D}'} \in supp(\bs{R}_{\bs{D}'}) : \forall\, V_i \in \bs{V}_{\bs{D}'},\; f_{V_i}\!\big(\bs{pa}_{V_i}(\bs{v}),\, \bs{r}_{\bs{D}'}\big) = v_i \big\},
 $$
where $\bs{pa}_{V_i}(\bs{v})$ denotes the values assigned to $V_i$'s parents by $\bs{v}$. For a single district $\bs{D}_k$, we write $h_k(\bs{v}):=h_{\bs{D}_k}(\bs{v})$.
\end{definition}
Intuitively, a district is a maximal set of variables tied together by shared latent confounding; the observational constraints couple the response-types within a district but factorize across districts (Lemmas~\ref{lemma:cartesian} and~\ref{lemma:district_cfactor}).

\begin{remark}
In Definition~\ref{def:district_compatible_restype}, $h_{\bs{D}'}(\bs{v})$ is a special case of $h(\tilde{\mathcal{a}})$ in Definition~\ref{def:compatible-restypes}, where (i) the argument is $\tilde{\mathcal{a}} = \bs{v}$, and (ii) the function is restricted to the response types in $\bs{D}'$, i.e., $\bs{R}_{\bs{D}'}$. Note that an observation can be viewed as a special type of experiment corresponding to an intervention on the empty set, $\emptyset$, yielding $\bs{V} = \bs{v}$ (in a single counterfactual world, i.e., $|\frak{C}| = 1$). 
Because every parent value is pinned by $\bs{v}$, the recursion $F^{(\mathfrak{c})}_{V'}$ collapses to a single  $f_{V'}$.
Moreover, it holds that $h(\bs{v}) = h_{\bs{D}}(\bs{v})$. However, in general, $h(\tilde{\mathcal{a}}) \neq h_{\bs{D}}(\tilde{\mathcal{a}})$ for an arbitrary experiment $\tilde{\mathcal{a}}$.
\end{remark}


Next, using the fact that $f_V$ depends only on the response-types of $V$'s district, we show that $h(\bs{v})$ can be decomposed into district-compatible response-types. 
See Appendix \ref{ex:cartesian} for an example. 

\begin{lemma}[Cartesian Product Decomposition]
\label{lemma:cartesian}
For any $\bs{v} \in supp(\bs{V})$: $h(\bs{v}) = \prod_{\bs{D}_k \in \bs{D}} h_k(\bs{v})$. 
\end{lemma}



\begin{definition}[C-Factor {\citep{tianPearlQfactor}}]
For district $\bs{D}_k$ and a topological order $V_1 \prec \cdots \prec V_n$, we define 
$
Q_{k}(\bs{v}):=Q_{\bs{D}_k}(\bs{v}) := \prod_{\{i: V_i \in \bs{D}_k\cap \bs{V}\}} P(v_i \mid \bs{v}^{(i-1)})
$
, where $\bs{V}^{(i-1)}:=\prod_{j<i}V_j$ and $\bs{V}^{(0)}=\emptyset$. The joint distribution factorizes as $P(\bs{v}) = \prod_{k =1}^c Q_k(\bs{v})$.

\end{definition}

The next lemma connects c-factors to response-type sums, providing an algebraic bridge needed to eliminate the irrelevant districts.

\begin{lemma}
\label{lemma:district_cfactor}
For any $\bs{D}' \subseteq \bs{D}$ and any $\bs{v} \in supp(\bs{V})$, we can decompose
$
\prod_{\bs{D}_k \in \bs{D}'} Q_{k}(\bs{v}) = \sum_{\bs{r}_{\bs{D}'} \in h_{\bs{D}'}(\bs{v})} P(\bs{r}_{\bs{D}'}).
$
\end{lemma}


\begin{definition}[Relevant Response-Types, $\bs{R}_\theta$, District Hull, $\bs{R}_\theta^*$]
$R^\star(\tilde{\mathcal{a}})$ is the set of response-types with at least one directed path to an outcome not passing through any intervened variables in any counterfactual world $\frak{c}\in \frak{C}$ of $\tilde{\mathcal{a}}$. We write $\bs{R}_\theta = R^\star(\theta)$ to denote the relevant response types of the query $\theta$. Let $\bs{D}^*_\theta$ denote the set of districts with non-empty overlap with $\bs{R}_\theta$. The \emph{district hull} is given by $D^*(\bs{R}_\theta):=\cup_{\bs{D}_k\in \bs{D}^*_\theta}\bs{D}_k\subseteq \bs{V}\cup \bs{R}$ and $\bs{R}_\theta^* = D^*(\bs{R}_\theta) \cap \bs{R}$.
\end{definition}
Intuitively, $\bs{R}_\theta$ contains the latent mechanisms that can still influence the query's outcomes once its interventions are in place; the hull adds every response-type that shares a district with one of them, since the observational constraints tie these together.
It has been shown by \cite[Proposition~3]{duarte2023automated} that the relevant response-types $R^\star(\tilde{\mathcal{a}})$ are sufficient to express the causal quantity $\tilde{\mathcal{a}}$ in the polynomial program. We can therefore formulate our objective function only in terms of $\bs{R}_\theta \subseteq \bs{R}_\theta^*$.

\begin{example}[Query-relevant response-types in Figure~\ref{fig:complexDAG}]
\label{ex:districts}
Let $\theta = P(Y \mid do(M))$ in the ADMG of Figure~\ref{fig:complexDAG}. To form $\bs{R}_\theta$, we need to find $R \in \bs{R}$ such that there is a directed path from $R$ to $Y$ not passing through $M$. $R_1\not\in \bs{R}_\theta$ since all paths $R_1 \to \dots \to Y$ pass through $M$. Similarly, $R_3$ and $R_4$ are not in $\bs{R}_\theta$. In contrast, $R_2 \to Y$ does not pass through $M$. Hence $\bs{R}_\theta = \{R_2\}$. The districts of Figure~\ref{fig:complexDAG} are $\bs{D}_1 = \{A, B, C, R_3, R_4\}$, $\bs{D}_2 = \{X, M, Y, R_1, R_2\}$, $\bs{D}_3 = \{D\}$. Since $R_2 \in \bs{D}_2$, then $\bs{D}^*_\theta = \{\bs{D}_2\}$. Consequently, we have $D^*(\bs{R}_\theta)=\bs{D}_2$ and $\bs{R}_\theta^* = \{R_1, R_2\}$.
\end{example}

Using Lemmas~\ref{lemma:cartesian} and~\ref{lemma:district_cfactor}, we can show that the observational constraints, which naively involve all of $\bs{R}$, can be reformulated to involve only $\bs{R}_\theta^*$. Specifically, the contribution of response-types outside the district hull factors out and cancels, leaving a set of constraints purely described by $\bs{R}_\theta^*$.

\begin{lemma}
\label{th-all_constr_on_Rtheta_in_Dtheta_c-factor}
For the purpose of optimizing $\mathcal{T}$, the observational constraints $P(\bs{v}) = \sum_{\bs{r} \in h(\bs{v})} P(\bs{r})$ can equivalently be replaced by
$
\prod_{\bs{D}_k \in \bs{D}^*_\theta}Q_k(\bs{v}) = \sum_{\bs{r}_\theta^* \in h_{\bs{D}^*_\theta}(\bs{v})} P(\bs{r}_\theta^*),
$
which involves only variables in $\bs{R}_\theta^*$. The left-hand side is computable from $P(\bs{V})$.
\end{lemma}

The base program can therefore be defined only over $\bs{R}_\theta^*$.

\begin{example}[Reducing observational constraints via c-factors in Figure~\ref{fig:complexDAG}]
\label{ex:cfactor_complexDAG}
Let $\theta = P(Y \mid do(M))$ in the ADMG of Figure~\ref{fig:complexDAG}.
Under the topological order $A \prec B \prec X \prec C \prec D \prec M \prec Y$, the c-factors of the districts in Example~\ref{ex:districts} are
$Q_{1}(\bs{v}) = P(a)\, P(b \mid a)\, P(c \mid a, b, x)$,
$Q_{2}(\bs{v}) = P(x \mid a, b)\, P(m \mid a, b, x, c, d)\, P(y \mid a, b, x, c, d, m)$, and
$Q_{3}(\bs{v}) = P(d \mid a, b, x, c)$.
By Lemma~\ref{th-all_constr_on_Rtheta_in_Dtheta_c-factor}, the full observational constraint $P(\bs{v}) = \sum_{\bs{r}\in h(\bs{v})} P(r_1,r_2,r_3,r_4)$ reduces to
$Q_{2}(\bs{v}) = \sum_{(r_1, r_2) \in h_{2}(\bs{v})} P(r_1, r_2)$:
$R_3$ and $R_4$ have been divided out, so the base program $\mathcal{P}_\emptyset$ optimizes over $|supp(R_1)| \cdot |supp(R_2)|$ decision variables rather than $\prod_{i=1}^4 |supp(R_i)|$.
\end{example}

\begin{theorem}
\label{th-doesnt_touch_r_theta_implies_pot0}
$\forall \tilde{\mathcal{a}} \in \mathcal{A}$ with $R^\star(\tilde{\mathcal{a}}) \cap \bs{R}_\theta^* = \emptyset$, $pot_\theta(\tilde{\mathcal{a}}) = 0$.
\end{theorem}

Theorem~\ref{th-doesnt_touch_r_theta_implies_pot0} reduces zero-potency certification to checking whether an experiment's relevant response-types overlap with $\bs{R}_\theta^*$. The following corollary gives a purely graphical sufficient condition for this check.

\begin{corollary}[Interception]
\label{th-experiments_with_intercepted_paths_useless}
Let $\tilde{\mathcal{a}} \in \mathcal{A}$ have counterfactual worlds $\frak{C}$. If for every $\frak{c} \in \frak{C}$, every directed path from $\bs{R}_\theta^*$ to $\bs{Y}^{(\frak{c})}$ passes through $\bs{X}^{(\frak{c})}$, then $pot_\theta(\tilde{\mathcal{a}}) = 0$.
\end{corollary}

\begin{example}[Zero potency]
\label{ex:zero-potency}
Consider experiment $P(D \mid do(C))$ in Example \ref{ex:districts}. The only directed path from $\bs{R}_\theta^*$ to $D$ is $R_1 \to C \to D$, which passes through $C$. Hence $pot_\theta(P(D \mid do(C))) = 0$. 
On the other hand, the experiment $P(D \mid do(B))$ has potentially non-zero potency, since the path $R_1 \to X \to C \to D$ does not pass through $B$.

\end{example}

Algorithm~\ref{alg:all_paths_intercepted} in Appendix~\ref{app:algorithms} determines whether a given single-world experiment is useless by verifying the condition of Corollary~\ref{th-experiments_with_intercepted_paths_useless}. 
If this algorithm returns \texttt{True}, the intervention is deemed useless; otherwise, the result is inconclusive. Its computational complexity is linear in the size of the graph. For a worked example on Figure~\ref{fig:complexDAG}, see Appendix~\ref{ex:pruning}.

\vspace{-.2cm}
\subsection{Identifiable Experiments}
\vspace{-.2cm}
When an intervention is identifiable, it cannot provide any additional information beyond what is already contained in the observational distribution. In other words, it is useless. 

\begin{wrapfigure}{r}{0.5\textwidth}
\vspace{-5pt}
\hrule\vspace{1pt}
\captionsetup{type=algorithm}
\caption{GetUselessExperiments}
\label{alg:get_useless_experiments}
\vspace{-4pt}
\hrule
\begin{algorithmic}
\Require Graph $\mathcal{G}$, Query $\theta$, Experiment set $\mathcal{A}$
\Ensure $\mathcal{A}^\dagger \subseteq \mathcal{A}$ with $\forall\, \tilde{\mathcal{a}} \in \mathcal{A}^\dagger: pot_\theta(\tilde{\mathcal{a}}) = 0$
\State $\bs{R}_\theta^* \gets D^*\!\big(R^\star(\theta)\big) \cap \bs{R}$; $\quad \mathcal{A}^\dagger \gets \emptyset$
\For{$\tilde{\mathcal{a}} \in \mathcal{A}$ with worlds $\frak{C}_{\tilde{\mathcal{a}}}$}
    \If{$\forall \frak{c}\! \in\! \frak{C}_{\tilde{\mathcal{a}}}\!\!:\!$ \textsc{Alg. \ref{alg:all_paths_intercepted}}$(\mathcal{G}, \bs{R}_\theta^*, \bs{Y}^{(\frak{c})}, \bs{X}^{(\frak{c})})$}
        \State $\mathcal{A}^\dagger \gets \mathcal{A}^\dagger \cup \{\tilde{\mathcal{a}}\}$; \textbf{continue}
    \EndIf
    \If{$|\frak{C}_{\tilde{\mathcal{a}}}| = 1$ \textbf{and} $\textsc{ID}(\mathcal{G}, \bs{Y}^{(\frak{c})}, \bs{X}^{(\frak{c})})$}
        \State $\mathcal{A}^\dagger \gets \mathcal{A}^\dagger \cup \{\tilde{\mathcal{a}}\}$
    \EndIf
\EndFor
\end{algorithmic}
\vspace{2pt}\hrule
\vspace{-20pt}
\end{wrapfigure}

\begin{theorem}
\label{id_implies_pot0}
$\forall\, \tilde{\mathcal{a}} \in \mathcal{A}$ where $P(\tilde{\mathcal{a}})$ is identifiable (\textsc{ID}) in $\mathcal{G}$, $pot_\theta(\tilde{\mathcal{a}}) = 0$.
\end{theorem}

The \textsc{ID} algorithm in \citet{shpitser2006identification} can be applied to verify identifiability for single-world interventions $P(\bs{w} \mid do(\bs{z}))$. Note that cross-world counterfactuals are not ID by construction, so \textsc{ID} is applied only to interventions. 
Furthermore, the results of Theorems \ref{th-doesnt_touch_r_theta_implies_pot0} and \ref{id_implies_pot0} can be applied independently, in any order, to filter out useless experiments, and Algorithm~\ref{alg:get_useless_experiments} exploits this fact. 

\begin{corollary}[ID-inertness]
\label{thm-id_inertness}
$\forall \tilde{\mathcal{a}}\in \mathcal{A}$ where $\tilde{\mathcal{a}}$ is \textsc{ID}: $pot_\theta(\{\tilde{\mathcal{a}}\} \cup \mathcal b)=pot_\theta(\mathcal b)$ for any $\mathcal b \subseteq \mathcal A$.
\end{corollary}
In words, ID-pruned experiments can be removed from the subset search
entirely without affecting the optimum
(Appendix~\ref{subsec:solving-max-potency}).

\begin{example}[Running Algorithm~\ref{alg:get_useless_experiments} on Figure~\ref{fig:complexDAG}]
\label{ex:alg_pruning}
Let $\theta = P(y \mid do(m))$ as in Example~\ref{ex:districts}, so $\bs{R}^*_\theta = \{R_1, R_2\}$, and consider the candidates
$\tilde{\mathcal{a}}_1\!: P(C, D \mid do(x))$,
$\tilde{\mathcal{a}}_2\!: P(Y, M, X \mid do(a, b))$,
$\tilde{\mathcal{a}}_3\!: P(Y, M \mid do(x))$,
$\tilde{\mathcal{a}}_4\!: P(Y, M \mid do(x'))$,
$\tilde{\mathcal{a}}_5\!: P(B \mid do(a))$, and
$\tilde{\mathcal{a}}_6\!: P(d_{c'}, d'_{c})$, i.e., testing for monotonicity along $C \to D$.
Interception (Algorithm \ref{alg:all_paths_intercepted}) certifies $\tilde{\mathcal{a}}_1$, $\tilde{\mathcal{a}}_5$ and $\tilde{\mathcal{a}}_6$ as useless: every directed path from $\{R_1, R_2\}$ to their outcomes passes through the intervened variable ($X$, $A$ and $C$, respectively). Moreover, $P(\tilde{\mathcal{a}}_2)$ is identifiable, so $\tilde{\mathcal{a}}_2$ is useless by Theorem~\ref{id_implies_pot0} and may be dropped from any subset search (Corollary~\ref{thm-id_inertness}). Algorithm~\ref{alg:get_useless_experiments} thus returns $\mathcal{A}^\dagger = \{\tilde{\mathcal{a}}_1, \tilde{\mathcal{a}}_2, \tilde{\mathcal{a}}_5, \tilde{\mathcal{a}}_6\}$, and only $\tilde{\mathcal{a}}_3$ and $\tilde{\mathcal{a}}_4$ need to be evaluated numerically.
\end{example}



%% file: Experiments.tex
\vspace{-.2cm}
\section{Evaluation}\label{sec:evaluation}
\vspace{-.2cm}
\textbf{Synthetic evaluation.}
We evaluate Algorithm~\ref{alg:get_useless_experiments} on two types of graphs: randomly generated graphs and several real-world networks selected from the \emph{bnlearn}\footnote{\url{https://www.bnlearn.com/bnrepository/}} repository. 
For each graph, we randomly choose a query as well as a set of $|\mathcal{A}|=200$ candidate experiments. We report the mean fraction of candidates pruned by the two core mechanisms of Algorithm~\ref{alg:get_useless_experiments}: the ID check, and the interception rule (Algorithm~\ref{alg:all_paths_intercepted}); throughout, $\pm$ denotes one standard deviation across simulations. For details on graph, query, and experiment generation as well as per-graph results we refer to Appendix~\ref{app:eval_setup}.

\textit{Erd\H{o}s--R\'enyi random graphs:}
We sample ADMGs with $|\bs{V}| \in \{10, 15, 20, 30\}$ observed nodes and vary the
confounding density $|\bs{U}|/|\bs{V}|$ from roughly $0.25$ to $1.0$. Across
$1{,}596$ simulations, Algorithm~\ref{alg:get_useless_experiments} prunes on
average $74.3\% \pm 19.4\%$ of candidate experiments ($33.0\% \pm 23.2\%$ via the ID check,
$41.3\% \pm 38.5\%$ via interception), reaching as much as $90.6\%$ in some simulations. Figure~\ref{fig:er_main} reveals a clear pattern: as $|\bs{U}|/|\bs{V}|$ grows,
interception weakens (from $70.9\% \pm 31.0\%$ at the sparsest setting to
$16.5\% \pm 30.0\%$ at the densest). The fraction pruned by the ID check, in turn, increases with confounding density, since the ID algorithm is only executed on experiments that were not already pruned by interception.

\begin{figure}[tb]
    \centering
    \includegraphics[width=.9\linewidth]{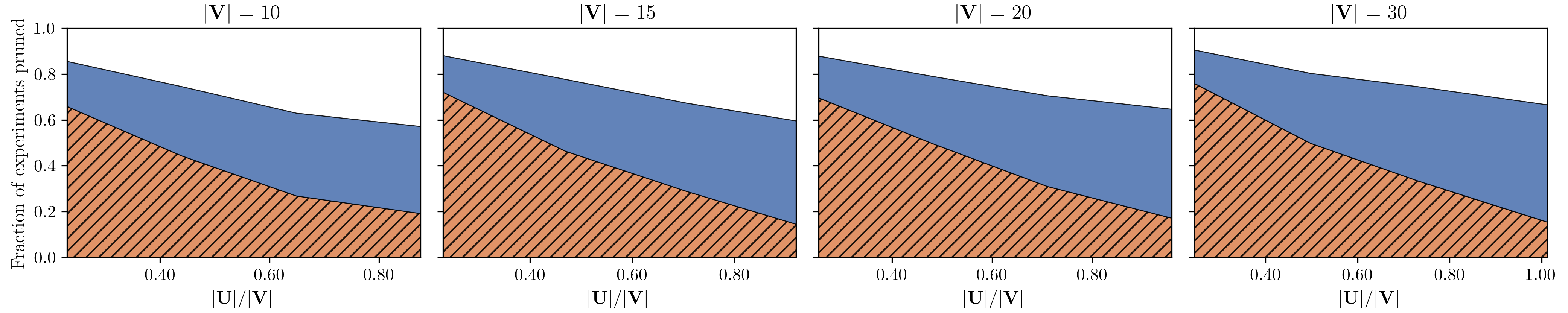}
    \caption{Erd\H{o}s--R\'enyi graphs. For different sizes of $\bs{V}$, we plot the fraction of pruned experiments against the density $\frac{|\bs{U}|}{|\bs{V}|}$. The hatched orange area represents experiments pruned by the interception mechanism (Algorithm~\ref{alg:all_paths_intercepted}); the blue area represents those pruned by the ID check. In total, $1{,}596$ simulations were performed.}
    \label{fig:er_main}
    \vspace{-.2cm}
\end{figure}

\textit{{bnlearn}:}
We evaluate on 11 benchmark networks ranging from $|\bs{V}|{=}8$
(\texttt{ASIA}) to $|\bs{V}|{=}76$ (\texttt{WIN95PTS}). Across $1{,}100$ simulations, Algorithm~\ref{alg:get_useless_experiments} prunes on average $63.4\% \pm 12.4\%$ of candidate experiments ($47.1\% \pm 14.5\%$ via the ID check, $16.2\% \pm 17.5\%$ via interception). Figure~\ref{fig:bnlearn_main} displays the pruning rates per graph. The contribution of the interception rule varies dramatically across graphs, ranging from $2.3\% \pm 4.3\%$ on \texttt{MILDEW} to $29.2\% \pm 25.5\%$ on \texttt{WIN95PTS}.

\begin{figure}[h]
    \centering
    \includegraphics[width=.8\linewidth]{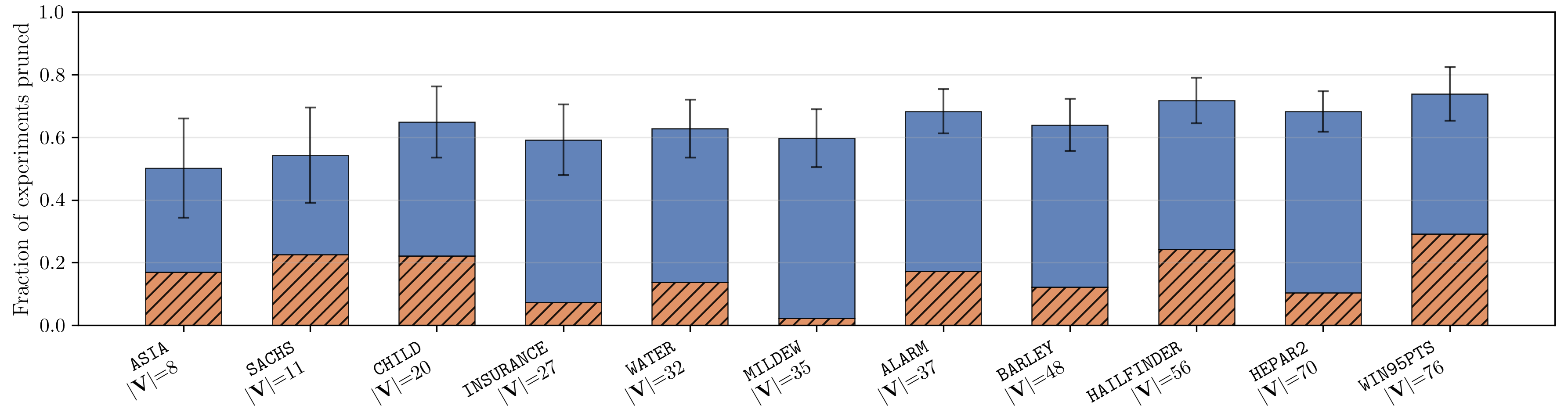}
    \caption{\emph{bnlearn}. For each graph, we plot the mean fraction of pruned experiments. The hatched orange area represents experiments pruned by the interception mechanism (Algorithm~\ref{alg:all_paths_intercepted}); the blue area represents those pruned by the ID check. Error bars indicate $\pm 1$ standard deviation of the total pruning rate across 100 simulations per graph. In total, $1{,}100$ simulations were performed.}
    \label{fig:bnlearn_main}
   \vspace{-.1cm}
\end{figure}

The high variability in the interception rate (e.g. $\pm 25.5\%$ on \texttt{WIN95PTS}) reflects its strong dependence on the specific topology. Figure~\ref{fig:bnlearn_Rstartheta} in Appendix \ref{app:eval_bnlearn} shows the main structural driver of
interception pruning.
When the query is localized (small $|\bs{R}_\theta^*|/|\bs{V}|$), most
experiments fail to interact with $\bs{R}_\theta^*$ and are pruned immediately;
as the hull grows, this mechanism weakens monotonically. This matches the
Erd\H{o}s--R\'enyi findings: denser confounding produces larger districts,
which enlarge $\bs{R}_\theta^*$ and shift the pruning burden onto the ID
check. Together, the two criteria yield substantial pruning across singleton candidates without solving a single polynomial program.

\textit{Subset search space reduction:}
The pruning rates reported above determine the search space
reduction when selecting a single optimal experiment, a practically
important setting. For
the general subset search over $2^{|\mathcal{A}|}$ subsets,
Corollary~\ref{thm-id_inertness} reduces the exponent by
$|\mathcal{A}^\dagger_{\mathrm{ID}}|$ (Appendix~\ref{subsec:solving-max-potency}). Running the ID check on all
candidates independently (without prior interception filtering), it
certifies $61.4\% \pm 11.3\%$ on Erd\H{o}s--R\'enyi and
$60.0\% \pm 10.8\%$ on \emph{bnlearn}---for $|\mathcal{A}|=200$, the
exponent drops from $200$ to roughly $80$ on average.

\textit{Runtime:}
Pruning inspects only the graph and the query and runs in milliseconds; what it saves is solver time. On $26$ random graphs with $|\bs{V}|=5$ (small enough for the solver to finish) and all $152$ single-variable interventions they support, pruning took about $6$\,ms per graph, while evaluating every candidate took at least $7.4$ hours of solver time against $1.1$ hours for the $48$ survivors (Appendix~\ref{subsec:runtime-pruning}), a speedup of at least $6.95\times$ that exceeds the drop in candidate count because the discarded candidates are also the slowest to solve (median $612$\,s versus $2.5$\,s). 

\begin{figure}[t]
\centering
\begin{minipage}[c]{0.40\textwidth}
\centering
\scalebox{0.9}{%
\begin{tikzpicture}[->, thick,
  obs/.style={draw, circle, minimum size=0.65cm, inner sep=1pt},
  lat/.style={draw, circle, dashed, minimum size=0.65cm, inner sep=1pt}]
  \node[obs] (A) at (0,0) {$A$};
  \node[obs] (B) at (1.7,0) {$B$};
  \node[obs] (X) at (3.4,0) {$X$};
  \node[obs] (Y) at (5.1,0) {$Y$};
  \node[obs] (Z) at (3.4,1.4) {$Z$};
  \node[lat] (U3) at (0.85,1.4) {$U_3$};
  \node[lat] (U1) at (2.1,1.4) {$U_1$};
  \node[lat] (U2) at (4.7,1.4) {$U_2$};
  \draw (A)--(B); \draw (B)--(X); \draw (A) to[bend right=35] (X);
  \draw (Z)--(X); \draw (X)--(Y);
  \draw[dashed] (U3)--(A); \draw[dashed] (U3)--(B);
  \draw[dashed] (U1)--(Z); \draw[dashed] (U1)--(X);
  \draw[dashed] (U2)--(X); \draw[dashed] (U2)--(Y);
\end{tikzpicture}}
\end{minipage}\hfill
\begin{minipage}[c]{0.56\textwidth}
\caption{Assumed causal ADMG for the NHANES experiment ($n=5{,}833$
complete cases). Solid nodes are observed; dashed nodes are latent
confounders.}
\label{fig:nhanes_dag}
\end{minipage}
\vspace{-.2cm}
\end{figure}

\textbf{Physical Activity and Diabetes.} We demonstrate the full pipeline, pruning followed by potency computation, on observational data from NHANES
2017--2018.\footnote{\url{https://wwwn.cdc.gov/nchs/nhanes/}}
 
\textit{Graph and query:}
Five binary variables describe the setting: balance problems
($A$), falls ($B$), health insurance ($Z$), physical activity ($X$),
and not diabetes ($Y$).\footnote{We invert $Y$ for better interpretability.} Three latent confounders could represent inner-ear
disorders ($U_3$, confounding $A \leftrightarrow B$), healthcare
access ($U_1$, confounding $Z \leftrightarrow X$), and genetic
predisposition ($U_2$, confounding $X \leftrightarrow Y$). The
assumed causal structure is shown in Figure~\ref{fig:nhanes_dag}; Appendix~\ref{sec:nhanes_details} reports experiment types, computational details, and robustness analyses.
 The query of interest is $\theta = \text{ATE} =P(y_x)-P(y_{x'})$, the average treatment effect of physical activity on not getting diabetes. The observational data yields a bound width of $\mathcal{W}_\theta=0.864$ with $\theta \in [-0.445, 0.419]$.
 
\textit{Experiments:} In practice, an analyst need not commit the entire research budget at once; a natural workflow is to select a small optimal batch, execute it, and re-solve with the realized outcome incorporated. We consider the candidate experiments and combinations in Table~\ref{tab:nhanes_results}. Algorithm~\ref{alg:get_useless_experiments} certifies $\tilde{\mathcal{a}}_1$ and $\tilde{\mathcal{a}}_2$ as useless.

\begin{table}
\centering\small
\caption{Pruning and potency results. Left: cost-1 singleton experiments ($\tilde{\mathcal{a}}_1, \tilde{\mathcal{a}}_2$ pruned by ID and Interception, respectively). Right: cost-2 singletons and combinations.}
\label{tab:nhanes_results}
\begin{minipage}[t]{0.48\textwidth}
\centering
\begin{tabular}{cllcc}
  \toprule
   & Experiment & Pruned? & $pot_\theta$ & cost \\
  \midrule
  $\tilde{\mathcal{a}}_1$ & $P(X \mid do(B))$    & \textbf{ID}   & $0$    & $1$ \\
  $\tilde{\mathcal{a}}_2$ & $P(B \mid do(A))$    & \textbf{Int.} & $0$    & $1$ \\
  $\tilde{\mathcal{a}}_3$ & $P(X,Y \mid do(z))$  & --            & $0$    & $1$ \\
  $\tilde{\mathcal{a}}_4$ & $P(y_{x'},y'_{x})$   & --            & $0.44$ & $1$ \\
  \bottomrule
\end{tabular}
\end{minipage}%
\hfill
\begin{minipage}[t]{0.48\textwidth}
\centering
\begin{tabular}{clcc}
  \toprule
   & Experiment & $pot_\theta$ & cost \\
  \midrule
  $\tilde{\mathcal{a}}_5$ & $P(Y \mid do(x'))$                & $0.36$ & $2$ \\
  $\tilde{\mathcal{a}}_6$ & $P(Y \mid do(x))$                 & $0.50$ & $2$ \\
                  & $\{\tilde{\mathcal{a}}_4, \tilde{\mathcal{a}}_6\}$ & $0.50$ & $3$ \\
                  & $\{\tilde{\mathcal{a}}_4, \tilde{\mathcal{a}}_5\}$ & $0.78$ & $3$ \\

  \bottomrule
\end{tabular}
\end{minipage}
\vspace{-.3cm}
\end{table}


After computing the potencies of the remaining experiments, we can see that $\tilde{\mathcal{a}}_3$ is useless, but was not pruned by either criterion. The solution to the max-potency problem depends on the budget $B$ committed per round.
For $B=1$, the optimal choice is $\{\tilde{\mathcal{a}}_4\}$. For $B=2$, we should carry out the positive intervention $\{\tilde{\mathcal{a}}_6\}$. For $B=3$, the most potent combination is $\{\tilde{\mathcal{a}}_4,\tilde{\mathcal{a}}_5\}$---even though the positive intervention ($\tilde{\mathcal{a}}_6$) is more potent by itself, in combination with $\tilde{\mathcal{a}}_4$, the negative intervention ($\tilde{\mathcal{a}}_5$) is superior.
Committing to $B=1$ and finding $P(y_{x'},y'_{x})=0$ (intuitively, physical activity cannot cause diabetes) would yield $\theta \in [0.074, 0.419]$, after which the analyst re-solves against the updated program.

%% file: appendix.tex
\clearpage
\appendix

\begin{center}
{\Large\textbf{Appendix}}
\end{center}

\section{Notation}
\label{app:notation}

Table~\ref{tab:notation} summarizes the notation used throughout the paper.
The Level column separates element-level from set-level symbols: a single
candidate experiment is written $\tilde{\mathcal{a}}$, a set of experiments
$\mathcal{a}$, and the candidate pool $\mathcal{A}$; for singletons we use the
shorthand of Definition~\ref{def:potency}.

\begin{table}[H]
\centering
\caption{Summary of notation.}
\label{tab:notation}
\footnotesize
\setlength{\tabcolsep}{4pt}
\begin{tabular}{@{}lll@{}}
\toprule
Symbol & Level & Meaning \\
\midrule
$X$; $x$ & --- & random variable; realized value ($x$/$x'$ for binary true/false) \\
$\bs{X}$, $\bs{x}$ & set & bold symbols denote sets of variables/values \\
$supp(X)$ & --- & support of $X$ \\
$\mathcal{M} = (\bs{V}, \bs{U}, \bs{F}, P(\bs{U}))$ & --- & structural causal model \\
$\mathcal{G} = (\bs{V}, E^d, E^b)$ & --- & ADMG: directed ($E^d$) and bidirected ($E^b$) edges \\
$\bs{pa}(V)$, $\bs{ch}(V)$, $\bs{anc}(V)$ & --- & parents, children, ancestors of $V$ \\
$do(\bs{X}=\bs{x})$; $P(\bs{Y}_{\bs{x}})$ & --- & intervention; post-intervention distribution \\
$\frak{C}$; $\frak{c}$ & --- & counterfactual worlds of a statement; a single world \\
$\bs{X}^{(\frak{c})}$, $\bs{Y}^{(\frak{c})}$ & set & intervened variables and outcomes of world $\frak{c}$ \\
$R_k$; $\bs{R}$; $\bs{r}$ & --- & response type of disturbance $U_k$; all response types; a configuration \\
$\theta$ & --- & target causal query \\
$h(\cdot)$ & --- & response-type configurations compatible with a statement \\
$R^\star(\cdot)$ & --- & relevant response types of a causal quantity \\
$\bs{R}_\theta$; $\bs{D}^*_\theta$; $\bs{R}^*_\theta$ & set & query-relevant response types; their districts; district hull \\
$Q_k$ & --- & c-factor of district $\bs{D}_k$ \\
\midrule
$\mathcal{A}$ & pool & set of all candidate experiments \\
$\tilde{\mathcal{a}} \in \mathcal{A}$ & element & a single candidate experiment \\
$\mathcal{a} \subseteq \mathcal{A}$ & set & a set of experiments \\
$\mathcal{A}^\dagger \subseteq \mathcal{A}$ & set & experiments certified useless by pruning \\
$p_{\tilde{\mathcal{a}}}$; $p_{\mathcal{a}}$ & element; set & outcome of one experiment; collection of outcomes \\
$[l_{p_{\tilde{\mathcal{a}}}}, u_{p_{\tilde{\mathcal{a}}}}]$ & element & feasible range of an outcome \\
$f_{\tilde{\mathcal{a}}}$ & element & constraint function of an experiment \\
$U_\theta$, $L_\theta$ & --- & observational upper/lower bound on $\theta$ \\
$U_\theta(\mathcal{a}, p_{\mathcal{a}})$, $L_\theta(\mathcal{a}, p_{\mathcal{a}})$ & set & bounds under experimental constraints \\
$\mathcal{W}_\theta(\mathcal{a})$; $\mathcal{W}_\theta$ & set & worst-case width; observational width \\
$pot_\theta(\mathcal{a})$ & set & epistemic potency (Definition~\ref{def:potency}) \\
$c(\cdot)$; $B$ & set & cost function on experiment sets; budget \\
$\mathcal{P}_\emptyset$; $\mathcal{P}_{\mathcal{a}}$ & --- & base and augmented polynomial programs \\
\bottomrule
\end{tabular}
\end{table}

\section{Definition of Response-Type Variables}\label{app:response-type-def}

Following \citet{duarte2023automated}, we define response-type variables per disturbance and assume the given ADMG $\mathcal{G}$ to be \emph{canonical} without loss of generality (see their paper for details on canonicalization).

\begin{definition}[Response-Type Variable]\label{def:response-type}
Let $\mathcal{G}$ be a canonical ADMG with discrete observed variables 
$\bs{V}$ and latent disturbances $\bs{U}$. For a disturbance $U_k \in 
\bs{U}$ with observed children $\bs{ch}(U_k) \subseteq \bs{V}$, declare 
two values $u_k, u_k' \in supp(U_k)$ equivalent, written $u_k \sim_k u_k'$, 
if for every $V \in \bs{ch}(U_k)$ the partially applied structural 
equation
\[
f_V^{(u_k)} : \prod_{P \in \bs{Pa}(V) \setminus \{U_k\}} supp(P) \;\to\; 
supp(V)
\]
is identical under $u_k$ and $u_k'$. The equivalence classes 
$supp(U_k)/\!\sim_k$ are the \emph{generalized principal strata} of 
$U_k$ \citep{frangakis2002principal, balke1997bounds}, and the 
\emph{response-type variable} $R_k$ associated with $U_k$ is the 
discrete categorical variable whose support is exactly these classes. 
The induced distribution $P(R_k)$ aggregates $P(U_k)$ over each class. 
We write $\bs{R} = \{R_k\}_{U_k \in \bs{U}}$ for the collection of all 
response-type variables.
\end{definition}

\begin{proposition}[\citealp{duarte2023automated}, Proposition 1]
Replacing each $U_k$ in the canonical SCM with its response-type 
variable $R_k$ yields an equivalent SCM with respect to the full data 
law (every factual, interventional, and counterfactual joint 
distribution is preserved).
\end{proposition}


\section{A Primer on the Polynomial-Programming Framework}
\label{app:primer}

This appendix gives a jargon-free walkthrough of the bounding framework of
\citet{duarte2023automated} on which our method builds, using the running
example of Section~\ref{sec:bounds}; Appendix~\ref{app:response-type-def}
gives the formal construction.

\paragraph{The idea.}
When a causal quantity cannot be determined exactly from the available data, we
instead ask: what is the smallest and largest value it could take while
remaining consistent with the data? Consider the PNS in
Figure~\ref{fig:confDAG}, where the binary variables $X$ and $Y$ share an
unobserved common cause $U$. Because $X$ and $Y$ are binary, $U$ can influence
them in only finitely many deterministic ways. It can fix $X$ to either $0$ or
$1$, and it can select one of four possible relationships between $X$ and $Y$:
$Y = 1$ regardless of $X$; $Y = 1$ exactly when $X = 1$ (the ``complier''
type); $Y = 1$ exactly when $X = 0$; and $Y = 0$ regardless of $X$. Thus $U$
can be replaced, without loss of generality, by a discrete response type with
$2 \times 4 = 8$ possible values ($2$ possible behaviors for $X$, $4$ for
$Y$). The population is then described by eight unknown probabilities, one for
each type, and every observational, interventional, or counterfactual quantity
can be written in terms of these probabilities.

\paragraph{Why this gives bounds.}
The PNS is the proportion of the population belonging to the complier type,
summed over the two possible response types for $X$. The observed distribution
$P(X, Y)$ restricts where this complier probability can be placed: a complier
has $Y = X$, so compliers may contribute to the observed cells
$(X,Y) = (1,1)$ and $(0,0)$, but not to $(1,0)$ or $(0,1)$. Therefore the total
complier probability cannot exceed $P(X{=}1, Y{=}1) + P(X{=}0, Y{=}0)$, which
is exactly the observational upper bound in Example~\ref{ex:obs_bounds}.
Bounding is therefore an optimization problem: maximize or minimize the causal
quantity of interest subject to the constraints imposed by the observed data.

\paragraph{What \citet{duarte2023automated} automate.}
Their method mechanizes these steps for general discrete causal graphs and
queries: it constructs the possible response types from the graph, expresses
the target query and the observed probabilities in terms of the response-type
probabilities, and sends the resulting minimization or maximization problem to
a global optimization solver such as Couenne. When the graph contains several
independent unobserved variables, the probability of a joint response-type
configuration is a \emph{product} of their individual probabilities; the
resulting optimization problem therefore contains polynomial, rather than only
linear, expressions and requires a global solver.

Our work builds on this framework: an experiment contributes additional
constraints of the same general form, as described in
Section~\ref{sec:formulating}, and potency measures how much these new
constraints narrow the resulting bounds.

\section{Alternative Criteria for $\mathcal{W}_\theta$}
\label{sec:alternative_W_cal}

In Section~\ref{sec:bounds}, we defined $\mathcal{W}_\theta(\mathcal{a})$ using the worst-case (maximum) over feasible experimental outcomes. Here, we briefly discuss two alternative criteria and compare them on our running example.

\paragraph{Option 1: Expected width.}
One could assume a distribution over the unknown experimental outcome, e.g. $p_{\tilde{\mathcal{a}}} \sim \mathrm{Uni}(l_{p_{\tilde{\mathcal{a}}}},u_{p_{\tilde{\mathcal{a}}}})$, and define
\[
\mathcal{W}_\theta(\tilde{\mathcal{a}}) = \mathbb{E}[U_\theta(\tilde{\mathcal{a}}, {p_{\tilde{\mathcal{a}}}}) -L_\theta(\tilde{\mathcal{a}}, {p_{\tilde{\mathcal{a}}}})]
= \frac{1}{u_{p_{\tilde{\mathcal{a}}}}-l_{p_{\tilde{\mathcal{a}}}}}\int_{l_{p_{\tilde{\mathcal{a}}}}}^{u_{p_{\tilde{\mathcal{a}}}}} U_\theta(\tilde{\mathcal{a}}, {p_{\tilde{\mathcal{a}}}}) -L_\theta(\tilde{\mathcal{a}}, {p_{\tilde{\mathcal{a}}}})\, dp
\]
for a singleton $\tilde{\mathcal{a}}\in \mathcal{A}$ as the expected width.

\paragraph{Option 2: Best-case width.}
One may change the maximum to a minimum in the original definition and define
\[
\mathcal{W}_\theta(\mathcal{a}) = \underset{\substack{p_{\mathcal{a}}\in\prod_{\tilde{\mathcal{a}} \in \mathcal{a}}[l_{p_{\tilde{\mathcal{a}}}}, u_{p_{\tilde{\mathcal{a}}}}]}}{\min} \bigl(U_\theta(\mathcal{a}, p_{\mathcal{a}}) - L_\theta(\mathcal{a}, p_{\mathcal{a}})\bigr)
\]
as the smallest possible bound width an experiment could yield.

\paragraph{Comparison on the running example.}
\begin{table}[H]
\centering
\caption{Comparison of alternative width criteria on the running example.}
\label{tab:alt_width}
\begin{tabular}{lcccc}
\toprule
Criterion 
& $\mathcal{W}_\theta$ 
& $\mathcal{W}_\theta(\tilde{\mathcal{a}}_1)$ 
& $\mathcal{W}_\theta(\tilde{\mathcal{a}}_2)$ 
& $\mathcal{W}_\theta(\mathcal{a}_3)$ \\
\midrule
Expected width 
& $0.6$ & $0.371$ & $0.467$ & $0.238$ \\
Best-case width 
& $0.6$ & $0.2$ & $0.4$ & $0.0$ \\
Worst-case width 
& $0.6$ & $0.5$ & $0.5$ & $0.4$ \\
\bottomrule
\end{tabular}
\end{table}

\noindent Table~\ref{tab:alt_width} displays alternative definitions of $\mathcal{W}_\theta$ under the setting of Example~\ref{ex:obs_bounds}. Notably, at $(p_{\tilde{\mathcal{a}}_1}, p_{\tilde{\mathcal{a}}_2}) = (0.2, 0.6)$, the bounds collapse to $\text{PNS} = 0$; the PNS is not identifiable from observational data alone, but becomes point-identified for this specific experiment realization.

Under all three criteria, $\mathcal{a}_3$ yields the smallest bound width; among singleton experiments, both $\tilde{\mathcal{a}}_1$ and $\tilde{\mathcal{a}}_2$ are equivalent under the worst-case criterion, while $\tilde{\mathcal{a}}_1$ is strictly preferred under the alternatives. Only the worst-case criterion provides a \emph{guaranteed} reduction in bound width regardless of the experimental outcome, which is why we adopt it throughout this paper.

\subsection{Criterion Sensitivity on the NHANES Candidates}
\label{subsec:criteria-nhanes}

To assess how much the choice of criterion matters in practice, we recomputed
the unpruned NHANES candidates of Table~\ref{tab:nhanes_results} under all
three criteria. Table~\ref{tab:criteria_nhanes} reports the resulting
potencies.\footnote{The Worst column is exact: it is the single coupled program
of Appendix~\ref{subsec:nhanes-computation}. The Expected and Best columns admit
no such collapse (see (ii) below), so they are obtained by sweeping the feasible
outcomes---$800$ points per scalar outcome dimension, $30 \times 30$ for outcome
pairs, and for $\tilde{\mathcal{a}}_3$ a sweep of $30$ sampled feasible
$do(Z{=}0)$ distributions against the feasible distributions on a $1/30$
enumeration grid for $do(Z{=}1)$---and are therefore subject to discretization
error.}

\begin{table}[H]
\centering\small
\caption{NHANES candidates under the worst-case, expected, and best-case
criteria (potency values; observational width $\mathcal{W}_\theta = 0.86$).}
\label{tab:criteria_nhanes}
\begin{tabular}{lcccc}
\toprule
Candidate & Worst & Expected & Best & Cost \\
\midrule
$\tilde{\mathcal{a}}_3 = P(X,Y \mid do(z))$   & $0$    & $0.49$ & $0.72$ & $1$ \\
$\tilde{\mathcal{a}}_4 = P(y_{x'}, y'_{x})$   & $0.44$ & $0.44$ & $0.52$ & $1$ \\
$\tilde{\mathcal{a}}_5 = P(Y \mid do(x'))$    & $0.36$ & $0.36$ & $0.36$ & $2$ \\
$\tilde{\mathcal{a}}_6 = P(Y \mid do(x))$     & $0.50$ & $0.50$ & $0.50$ & $2$ \\
$\{\tilde{\mathcal{a}}_4, \tilde{\mathcal{a}}_6\}$ & $0.50$ & $0.51$ & $0.55$ & $3$ \\
$\{\tilde{\mathcal{a}}_4, \tilde{\mathcal{a}}_5\}$ & $0.78$ & $0.78$ & $0.83$ & $3$ \\
\bottomrule
\end{tabular}
\end{table}

The criterion changes the \emph{decision}, not merely the numbers: at $B = 1$,
the worst-case criterion selects $\tilde{\mathcal{a}}_4$, whereas the expected
and best-case criteria both select $\tilde{\mathcal{a}}_3$---an experiment
whose worst-case potency is zero, so it cannot be guaranteed to narrow the
bounds at all, yet which reaches $0.49$ on average over its outcomes and $0.72$
at its most favorable one. At $B = 3$,
$\{\tilde{\mathcal{a}}_4, \tilde{\mathcal{a}}_5\}$ is optimal under all three
criteria among the sets evaluated.

That said, the criterion is a modular choice within our framework.
(i) The pruning results hold under all three criteria: our proofs show
something stronger than zero worst-case potency---for pruned experiments, the
bounds remain unchanged under \emph{every} feasible outcome
(Theorems~\ref{th-doesnt_touch_r_theta_implies_pot0} and~\ref{id_implies_pot0}).
Pruning is therefore unaffected by the criterion; only the ranking among the
survivors depends on it.
(ii) For potency evaluation, the worst-case criterion is what allows collapsing
the whole computation into a single polynomial program
(Appendix~\ref{subsec:evaluating-potency}), and the collapse is specific to it.
The worst case maximizes a difference between a maximum and a minimum, so all
three optimizations merge into one; the best case replaces the outer maximum by
a minimum and is a min--max, while the expected width is an integral
over outcomes rather than an optimization at all. Under either alternative one
must sweep the feasible outcomes and solve a separate program at each: the
Expected and Best columns of Table~\ref{tab:criteria_nhanes} take hundreds to
thousands of programs per candidate, against one for the Worst column.

\subsection{Beyond a Single Budget: Money and Time}
\label{subsec:two-budget}

The budget model is equally modular. Practitioners may care about resources
beyond money, such as the time needed to conduct an experiment. Concretely, let
each experiment set carry a pair of costs $(c(\mathcal{a}), t(\mathcal{a}))$,
money and time, with two budgets $B_c$ and $B_t$. The max-potency problem
becomes
$\max_{\mathcal{a} \subseteq \mathcal{A}} pot_\theta(\mathcal{a})$ subject to
$c(\mathcal{a}) \leq B_c$ \emph{and} $t(\mathcal{a}) \leq B_t$. Costs never
enter the potency computation itself; they only decide which sets are
affordable, so the only change to the solution method is that the single budget
check per candidate set becomes double. The pruning procedure is likewise
unaffected: it uses the graph, the query, and the candidate list, never the
costs---a zero-potency experiment is useless no matter how cheap or fast it is.
Since the original problem is the special case $B_t = \infty$, the extension
remains NP-hard, and the computational bottleneck is unchanged. Note that our cost functions are defined over
\emph{sets} of experiments, so time need not be additive---experiments
performed in parallel cost the maximum of their durations rather than the sum,
and the formulation captures this as-is. In the sequential setting, which we
leave as future work, time becomes even more relevant: for example, one may
seek to reach a target bound width within a given deadline at least cost.

\section{NP-hardness of the Max-Potency Problem}
\label{app:np-hardness}

In this appendix we prove Theorem~\ref{thm:np-hard} (max-potency is
NP-hard) via a polynomial-time reduction from the 0-1 knapsack problem.

At a high level, the reduction maps every knapsack object to an experiment: the
object's weight becomes the experiment's cost, and the object's value becomes
the experiment's potency. The gadget that makes this work is the
\emph{TP-cell} defined below: a minimal ADMG with a treatment $X_i$, an outcome
$Y_i$, and confounding between them, in which the query of interest is the
$\mathrm{PNS}_i$ with its symbolic bounds from \citet{tian_probabilities_2000}.
Given a knapsack instance with values $v_i$, weights $w_i$, and budget $B$, we
build one TP-cell per item, choose each cell's observational distribution so
that the potency of the cell's single experiment
$\tilde{\mathcal{a}}_i = P(y_{i,x_i})$ encodes the item's value
(Lemma~\ref{lem:achievability}; the formal proof first rescales the values into
the achievable range, which does not change the optimal knapsack solution),
take the query to be the sum of all $\mathrm{PNS}_i$, and assign
$\tilde{\mathcal{a}}_i$ cost $w_i$. Because distinct cells share no variables,
performing several experiments tightens each cell's bounds independently, so
the potency of a set of experiments is simply the sum of the per-cell potencies
(Lemma~\ref{lem:additivity})---just like the value of a set of packed objects.
Solving the constructed max-potency instance therefore solves the knapsack
instance, and each step of the construction runs in polynomial time.

\subsection{The atomic unit: TP-cells}

\begin{definition}[TP-cell]
\label{def:tp-cell}
A \emph{TP-cell} $\mathcal{C}_i$ is the ADMG with binary observed
variables $X_i, Y_i$, a directed edge $X_i \to Y_i$, and a bidirected
edge $X_i \leftrightarrow Y_i$, equipped with an observational
distribution $P_i(X_i, Y_i)$ parameterized by
\[
\alpha_i = P_i(x_i, y_i), \quad
\beta_i  = P_i(x_i, y_i'), \quad
\gamma_i = P_i(x_i', y_i), \quad
\delta_i = P_i(x_i', y_i'),
\]
with $(\alpha_i, \beta_i, \gamma_i, \delta_i) \in \Delta$, where
\[
\Delta := \bigl\{(\alpha, \beta, \gamma, \delta) \in
\mathbb{R}_{\geq 0}^{4} :
\alpha + \beta + \gamma + \delta = 1\bigr\}
\]
denotes the probability simplex over the four joint outcomes of
$(X_i, Y_i)$.
\end{definition}

\begin{figure}[h]
\centering
\begin{tikzpicture}[
  >=Latex,
  node distance=14mm,
  every node/.style={inner sep=1.5pt}
]
  \node (X) {$X_i$};
  \node (Y) [right=of X] {$Y_i$};
  \draw[->] (X) -- (Y);
  \draw[<->, dashed, bend left=40] (X) to (Y);
\end{tikzpicture}
\caption{A TP-cell $\mathcal{C}_i$: a confounded treatment-outcome pair.}
\label{fig:tp-cell}
\end{figure}

The query of interest within a single cell is the probability of
necessity and sufficiency,
\[
\mathrm{PNS}_i = P(y_{i,x_i}, y'_{i,x_i'}).
\]
For a TP-cell, \citet{tian_probabilities_2000} give the
sharp observational bounds
\[
0 \leq \mathrm{PNS}_i \leq \alpha_i+\delta_i,
\]
so that the observational worst-case width is
\[
\mathcal{W}_i(\emptyset)=\alpha_i+\delta_i.
\]

We write $pot_i(\tilde{\mathcal{a}}_i) = \mathcal{W}_i(\emptyset) - \mathcal{W}_i(\tilde{\mathcal{a}}_i)$
for the per-cell potency, in line with Definition~\ref{def:potency}. Furthermore, we use $W_{(\cdot)}(\mathcal{a}, p_{\mathcal{a}}) := U_{(\cdot)}(\mathcal{a}, p_{\mathcal{a}}) - L_{(\cdot)}(\mathcal{a}, p_{\mathcal{a}})$ to denote the instantiated bound width.

\subsection{The construction}
\label{app:np-hardness_construction}

Given a 0-1 knapsack instance with values $v_1, \dots, v_n \geq 0$,
weights $w_1, \dots, w_n \geq 0$, and budget $B \geq 0$, we construct
a max-potency instance as follows.

\begin{enumerate}
\item \textbf{Graph.} Let $\mathcal{G}$ be the disjoint union of $n$
TP-cells $\mathcal{C}_1, \dots, \mathcal{C}_n$. The observed variables
are $\bs{V} = \{X_1, Y_1, \dots, X_n, Y_n\}$.
\item \textbf{Observational distribution.} Choose per-cell distributions
$P_i(X_i, Y_i)$ such that $pot_i(\tilde{\mathcal{a}}_i) = v_i$, where
$\tilde{\mathcal{a}}_i = P(y_{i,x_i})$. Constructive achievability is established in
Lemma~\ref{lem:achievability} below; if necessary, rescale the values
$v_i$ so that they lie in the achievable range. Set
$P(\bs{V}) = \prod_{i=1}^{n} P_i(X_i, Y_i)$.
\item \textbf{Query.} Let
\[
\theta = \sum_{i=1}^{n} \mathrm{PNS}_i
       = \sum_{i=1}^{n} P(y_{i,x_i}, y'_{i,x_i'}).
\]
\item \textbf{Experiment set.} Let $\mathcal{A} = \{\tilde{\mathcal{a}}_1, \dots, \tilde{\mathcal{a}}_n\}$
with $\tilde{\mathcal{a}}_i = P(y_{i,x_i})$. For any $S\subseteq\{1,\dots,n\}$, we write $\mathcal{a}_S:=\{\tilde{\mathcal{a}}_i:i \in S\}$.
\item \textbf{Cost function and budget.} Set $c(\mathcal{a}_S) = \sum_{i \in S} w_i$
for $S \subseteq \{1, \dots, n\}$, and use the original budget $B$.
\end{enumerate}

The construction is polynomial in $n$: each TP-cell is a constant-size
object, and Lemma~\ref{lem:achievability} produces the required
$P_i$ in time polynomial in the bit-length of $v_i$.

\subsection{Additivity of potency across cells}

The key technical claim is that the potency of any subset $\mathcal{a}_S$
decomposes as a sum of per-cell potencies. This follows from the
district structure of $\mathcal{G}$ via the machinery developed in
Section~\ref{sec:pruning}.

\begin{lemma}[Additivity]
\label{lem:additivity}
For any $S \subseteq \{1, \dots, n\}$,
\[
pot_\theta(\mathcal{a}_S) \;=\; \sum_{i \in S} pot_i(\tilde{\mathcal{a}}_i).
\]
\end{lemma}

\begin{proof}
Each TP-cell $\mathcal{C}_i$ is a connected component of the bidirected
graph and shares no variables with any other cell, so the districts
of $\mathcal{G}$ are exactly $\mathcal{C}_1, \dots, \mathcal{C}_n$.
Let $R_i$ denote the response-types of cell $i$. By
Lemma~\ref{lemma:cartesian} (Cartesian product decomposition),
$h(\bs{v}) = \prod_{i=1}^{n} h_i(v_i)$ for any
$\bs{v} = (v_1, \dots, v_n) \in \mathrm{supp}(\bs{V})$, and
response-types in distinct cells are independent:
$P(\bs{r}) = \prod_{i=1}^{n} P(r_i)$.

We now show that the polynomial program $\mathcal{P}_{\mathcal{a}_S}$ decomposes
into $n$ independent per-cell programs.

\paragraph{Objective separates.}
Each $\mathrm{PNS}_i$ involves only interventions, outcomes, and
response-types of cell $i$, so the target polynomial
$\mathcal{T} = \sum_i \mathrm{PNS}_i$ is a sum where the $i$-th
summand depends only on $R_i$.

\paragraph{Observational constraints separate.}
The joint constraint $P(\bs{v}) = \sum_{\bs{r} \in h(\bs{v})} P(\bs{r})$
is implied by the per-cell constraints
$P_i(v_i) = \sum_{r_i \in h_i(v_i)} P(r_i)$ for $i = 1, \dots, n$,
since $P(\bs{v}) = \prod_i P_i(v_i)$ and $h(\bs{v}) = \prod_i h_i(v_i)$.

\paragraph{Experimental constraints separate.}
Each constraint $f_{\tilde{\mathcal{a}}_i}(\bs{R}) = p_{\tilde{\mathcal{a}}_i}$ involves only $R_i$.

Hence $\mathcal{P}_{\mathcal{a}_S}$ decomposes as
\[
U_\theta(\mathcal{a}_S, p_{\mathcal{a}_S})
   = \sum_{i \in S} U_i(\tilde{\mathcal{a}}_i, p_{\tilde{\mathcal{a}}_i})
   + \sum_{i \notin S} U_i(\emptyset),
\]
and similarly for $L_\theta$, giving the instantiated width
\[
W_\theta(\mathcal{a}_S, p_{\mathcal{a}_S})
   = \sum_{i \in S} W_i(\tilde{\mathcal{a}}_i, p_{\tilde{\mathcal{a}}_i})
   + \sum_{i \notin S} W_i(\emptyset).
\]

The worst-case is taken over the product domain
$p_{\mathcal{a}_S} \in \prod_{i \in S}[l_{p_{\tilde{\mathcal{a}}_i}}, u_{p_{\tilde{\mathcal{a}}_i}}]$. Since
the sum is separable in the $p_{\tilde{\mathcal{a}}_i}$, the maximum decomposes:
\[
\mathcal{W}_\theta(\mathcal{a}_S)
   = \max_{p_{\mathcal{a}_S}} W_\theta(\mathcal{a}_S, p_{\mathcal{a}_S})
   = \sum_{i \in S} \mathcal{W}_i(\tilde{\mathcal{a}}_i)
   + \sum_{i \notin S} W_i(\emptyset).
\]
For $S = \emptyset$ this gives
$\mathcal{W}_\theta(\emptyset) = \sum_i W_i(\emptyset)$.
Combining,
\begin{align*}
pot_\theta(\mathcal{a}_S)
&= \mathcal{W}_\theta(\emptyset) - \mathcal{W}_\theta(\mathcal{a}_S) \\
&= \sum_i W_i(\emptyset)
   - \sum_{i \in S} \mathcal{W}_i(\tilde{\mathcal{a}}_i)
   - \sum_{i \notin S} W_i(\emptyset) \\
&= \sum_{i \in S} \bigl[W_i(\emptyset) - \mathcal{W}_i(\tilde{\mathcal{a}}_i)\bigr]
 = \sum_{i \in S} pot_i(\tilde{\mathcal{a}}_i). \qedhere
\end{align*}
\end{proof}

\subsection{Achievability of target potencies}

\begin{lemma}[Constructive achievability]
\label{lem:achievability}
There exists $v^{*} > 0$ and a polynomial-time algorithm that, given
any rational $v \in [0, v^{*}]$, outputs rational parameters
$(\alpha, \beta, \gamma, \delta) \in \Delta$ with
$pot_i(\tilde{\mathcal{a}}_i) = v$.
\end{lemma}

\begin{proof}
We exhibit a one-parameter family of cells indexed directly by $v$.
For $v \in [0,\, 1/8]$, define
\[
P^{v} := \bigl(\alpha,\,\beta,\,\gamma,\,\delta\bigr)
       = \bigl(1 - 8v,\; v,\; 3v,\; 4v\bigr).
\]
All four entries are non-negative on $[0,\,1/8]$ and sum to $1$, so
$P^{v} \in \Delta$. We will show $pot_i(\tilde{\mathcal{a}}_i;\,P^{v}) = v$ on
this range, and set $v^{*} := 0.1$.

\paragraph{Observational width.}
By the Tian--Pearl observational bound stated above,
\begin{equation}
\label{eq:Wobs-Pv}
W_i(\emptyset;\,P^{v}) = \alpha + \delta = 1 - 4v.
\end{equation}

\paragraph{Width with the experimental constraint.}
We now use the Tian--Pearl bounds for PNS when the interventional
marginal \(p=P(y_{i,x_i})\) is known. With the experimental constraint
\(\tilde{\mathcal{a}}_i=P(y_{i,x_i})=p\) added to the observational data, the bounds still
involve the unobserved interventional marginal
\[
q := P(y_{i,x_i'}).
\]
Since \(q\) is not part of the allowed experiment set in the reduction,
it is optimized out over its observationally feasible range
\[
q\in[\gamma,1-\delta]=[3v,1-4v].
\]
Thus the sharp lower and upper bounds on \(\mathrm{PNS}_i\), conditional
on observing \(p\), are obtained as
\begin{align*}
L_i(\tilde{\mathcal{a}}_i,\,p)
&= \min_{q \in [3v,\,1-4v]}
   \max\bigl\{0,\; p - q,\; P(y) - q,\; p - P(y)\bigr\}, \\
U_i(\tilde{\mathcal{a}}_i,\,p)
&= \max_{q \in [3v,\,1-4v]}
   \min\bigl\{p,\; 1 - q,\; \alpha + \delta,\;
              p - q + \beta + \gamma\bigr\}.
\end{align*}
Substituting the marginals $P(y) = \alpha + \gamma = 1 - 5v$,
$\alpha + \delta = 1 - 4v$, and $\beta + \gamma = 4v$:

\medskip\noindent
\emph{Lower bound.}
The expression
$\max\{0,\, p - q,\, (1-5v) - q,\, p - (1-5v)\}$
is non-increasing in $q$ (each term is either independent of $q$ or
non-increasing in $q$), so the minimum over $q \in [3v,\,1-4v]$ is
attained at $q = 1 - 4v$. Substituting,
\[
L_i(\tilde{\mathcal{a}}_i,\,p)
= \max\bigl\{0,\; p + 4v - 1,\; -v,\; p + 5v - 1\bigr\}
= \max\bigl\{0,\; p + 5v - 1\bigr\},
\]
where we drop $-v \leq 0$ and note $p + 5v - 1 \geq p + 4v - 1$.

\medskip\noindent
\emph{Upper bound.}
The expression $\min\{p,\, 1 - q,\, 1 - 4v,\, p - q + 4v\}$ is
non-decreasing in $(-q)$, so the maximum over $q \in [3v,\,1-4v]$ is
attained at $q = 3v$. Substituting,
\[
U_i(\tilde{\mathcal{a}}_i,\,p)
= \min\bigl\{p,\; 1 - 3v,\; 1 - 4v,\; p + v\bigr\}
= \min\bigl\{p,\; 1 - 4v\bigr\},
\]
since $1 - 4v \leq 1 - 3v$ and $p \leq p + v$.

\medskip\noindent
The bound width $W_i(\tilde{\mathcal{a}}_i,\,p;\,P^{v}) = U_i(\tilde{\mathcal{a}}_i,\,p) - L_i(\tilde{\mathcal{a}}_i,\,p)$
follows by case analysis on $p \in [1 - 8v,\, 1 - v]$:
\begin{equation}
\label{eq:Wpointwise-Pv}
W_i(\tilde{\mathcal{a}}_i,\,p;\,P^{v}) =
\begin{cases}
p, & p \in [\,1 - 8v,\; 1 - 5v\,], \\[2pt]
1 - 5v, & p \in [\,1 - 5v,\; 1 - 4v\,], \\[2pt]
(2 - 9v) - p, & p \in [\,1 - 4v,\; 1 - v\,].
\end{cases}
\end{equation}
The function $p \mapsto W_i(\tilde{\mathcal{a}}_i,\,p;\,P^{v})$ is continuous and
piecewise-linear: it rises with slope $+1$ from $W = 1 - 8v$ to
$W = 1 - 5v$, plateaus at $1 - 5v$ on $[1-5v,\,1-4v]$, then falls
with slope $-1$ to $W = 1 - 8v$. The maximum is attained on the
plateau, giving
\begin{equation}
\label{eq:Wmax-Pv}
\mathcal{W}_i(\tilde{\mathcal{a}}_i;\,P^{v})
   = \max_{p \in [1-8v,\,1-v]} W_i(\tilde{\mathcal{a}}_i,\,p;\,P^{v})
   = 1 - 5v.
\end{equation}

\paragraph{Potency.}
Combining \eqref{eq:Wobs-Pv} and \eqref{eq:Wmax-Pv},
\[
pot_i(\tilde{\mathcal{a}}_i;\,P^{v})
   = W_i(\emptyset;\,P^{v}) - \mathcal{W}_i(\tilde{\mathcal{a}}_i;\,P^{v})
   = (1 - 4v) - (1 - 5v)
   = v.
\]
At $v = 0.1$ this recovers $pot_i(\tilde{\mathcal{a}}_i;\,P^{0.1}) = 0.1$ in
agreement with Example~\ref{ex:potency}.

\paragraph{Polynomial time.}
Given a rational $v$ with bit-length $L$, the four entries
$(1 - 8v,\, v,\, 3v,\, 4v)$ of $P^{v}$ are computed by a constant
number of rational arithmetic operations and have bit-length $\mathcal{O}(L)$.
The cell is therefore constructed in time polynomial in the
bit-length of $v$.
\end{proof}

\subsection{The reduction}

\begin{theorem}[Theorem~\ref{thm:np-hard}, restated]
\label{thm:np-hard-restated}
The max-potency problem is NP-hard.
\end{theorem}

\begin{proof}
We exhibit a polynomial-time reduction from 0-1 knapsack to
max-potency. Let $(v_1, \dots, v_n;\, w_1, \dots, w_n;\, B)$ be a
0-1 knapsack instance.

\paragraph{Rescaling.}
If $\max_i v_i > v^{*}$, replace each $v_i$ with
$\tilde{v}_i := v_i \cdot v^{*} / \max_j v_j$, so that $\tilde{v}_i
\in [0, v^{*}]$. Since this rescaling multiplies every subset sum
$\sum_{i \in S} v_i$ by the same positive constant
$v^{*} / \max_j v_j$, it preserves the argmax: the optimal $S^{*}$
under $(\tilde{v}, w, B)$ is identical to the optimal $S^{*}$ under
$(v, w, B)$. The rescaled values $\tilde{v}_i$ are rational with
bit-length polynomial in the original input. Henceforth we work
with $\tilde{v}_i$ in place of $v_i$.

\paragraph{Construction.}
Construct the max-potency instance
$(\mathcal{G}, P(\bs{V}), \theta, \mathcal{A}, c, B)$ as in
Section~\ref{app:np-hardness_construction}, with cell distributions chosen via
Lemma~\ref{lem:achievability} so that
$pot_i(\tilde{\mathcal{a}}_i) = \tilde{v}_i$. The construction runs in time
polynomial in the size of the knapsack instance.

\paragraph{Equivalence.}
By Lemma~\ref{lem:additivity}, for any $S \subseteq \{1, \dots, n\}$,
\[
pot_\theta(\mathcal{a}_S) = \sum_{i \in S} pot_i(\tilde{\mathcal{a}}_i)
                         = \sum_{i \in S} \tilde{v}_i,
\]
and by construction $c(\mathcal{a}_S) = \sum_{i \in S} w_i$. The max-potency
problem on the constructed instance is therefore
\[
\max_{S \subseteq \{1,\dots,n\}} \sum_{i \in S} \tilde{v}_i
\quad \text{s.t.} \quad \sum_{i \in S} w_i \leq B,
\]
whose argmax coincides, by the rescaling argument, with the argmax
of the original 0-1 knapsack instance. Since 0-1 knapsack is
NP-hard, max-potency is NP-hard.
\end{proof}
\section{Additional Examples}\label{app:example}
\begin{figure}[h]
\centering
\begin{subfigure}{0.32\textwidth}
  \centering
  \includegraphics[width=\linewidth]{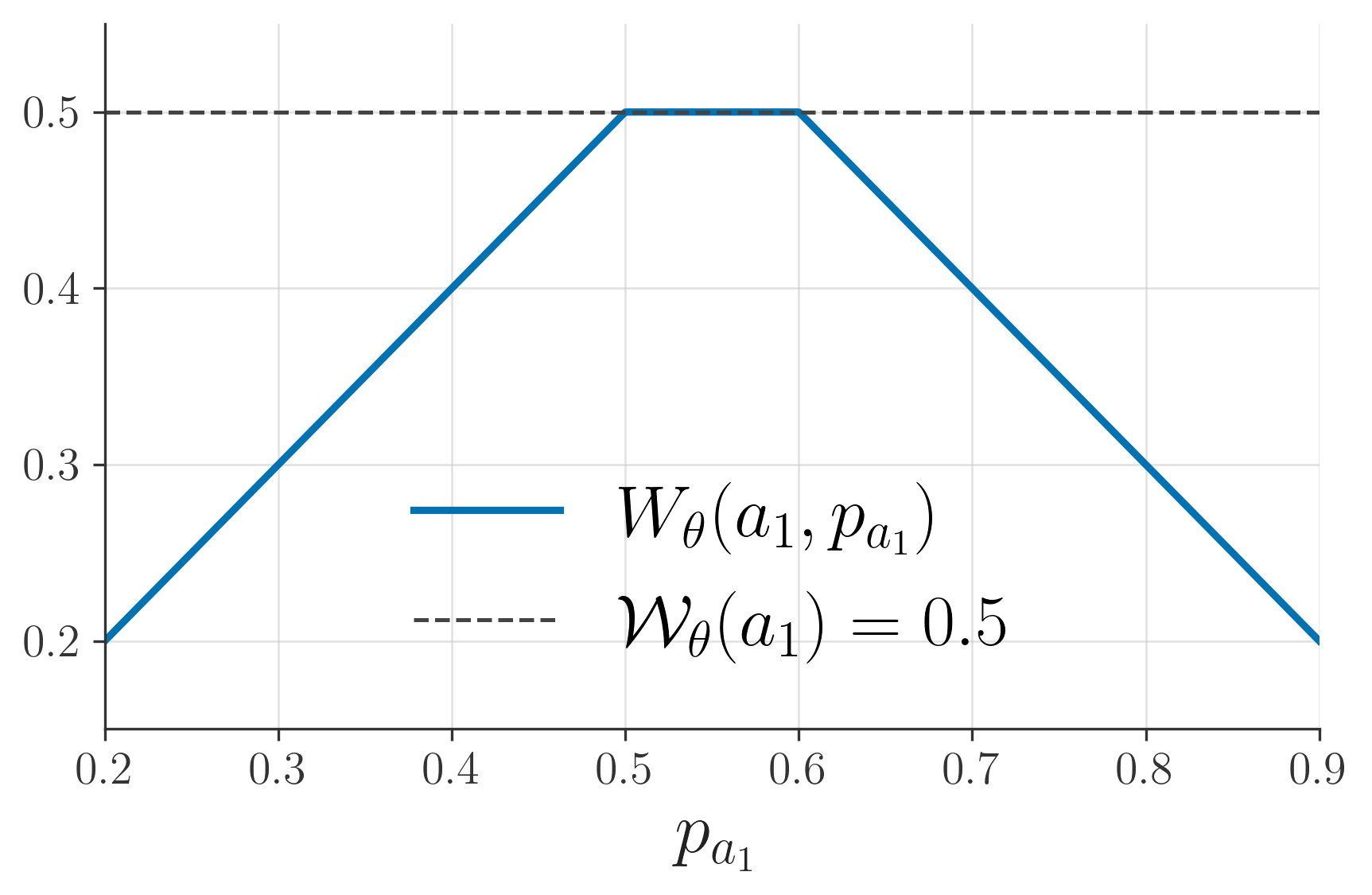}
  \caption{Experiment $\tilde{\mathcal{a}}_1$. }
  \label{fig:a1_W}
\end{subfigure}
\hfill
\begin{subfigure}{0.32\textwidth}
  \centering
  \includegraphics[width=\linewidth]{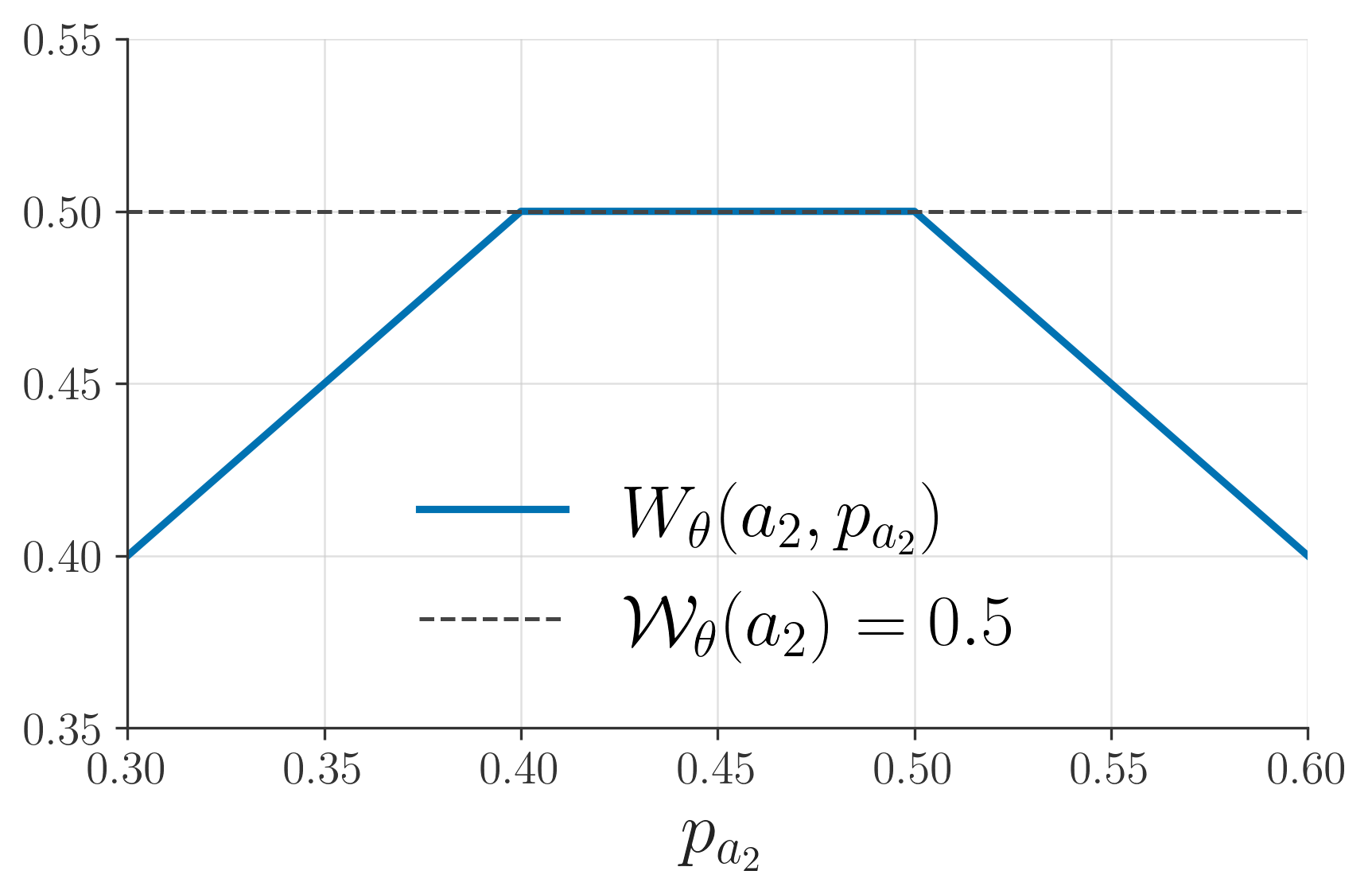}
  \caption{Experiment $\tilde{\mathcal{a}}_2$.}
  \label{fig:a2_W}
\end{subfigure}
\hfill
\begin{subfigure}{0.32\textwidth}
  \centering
  \includegraphics[width=\linewidth, height=3.5cm]{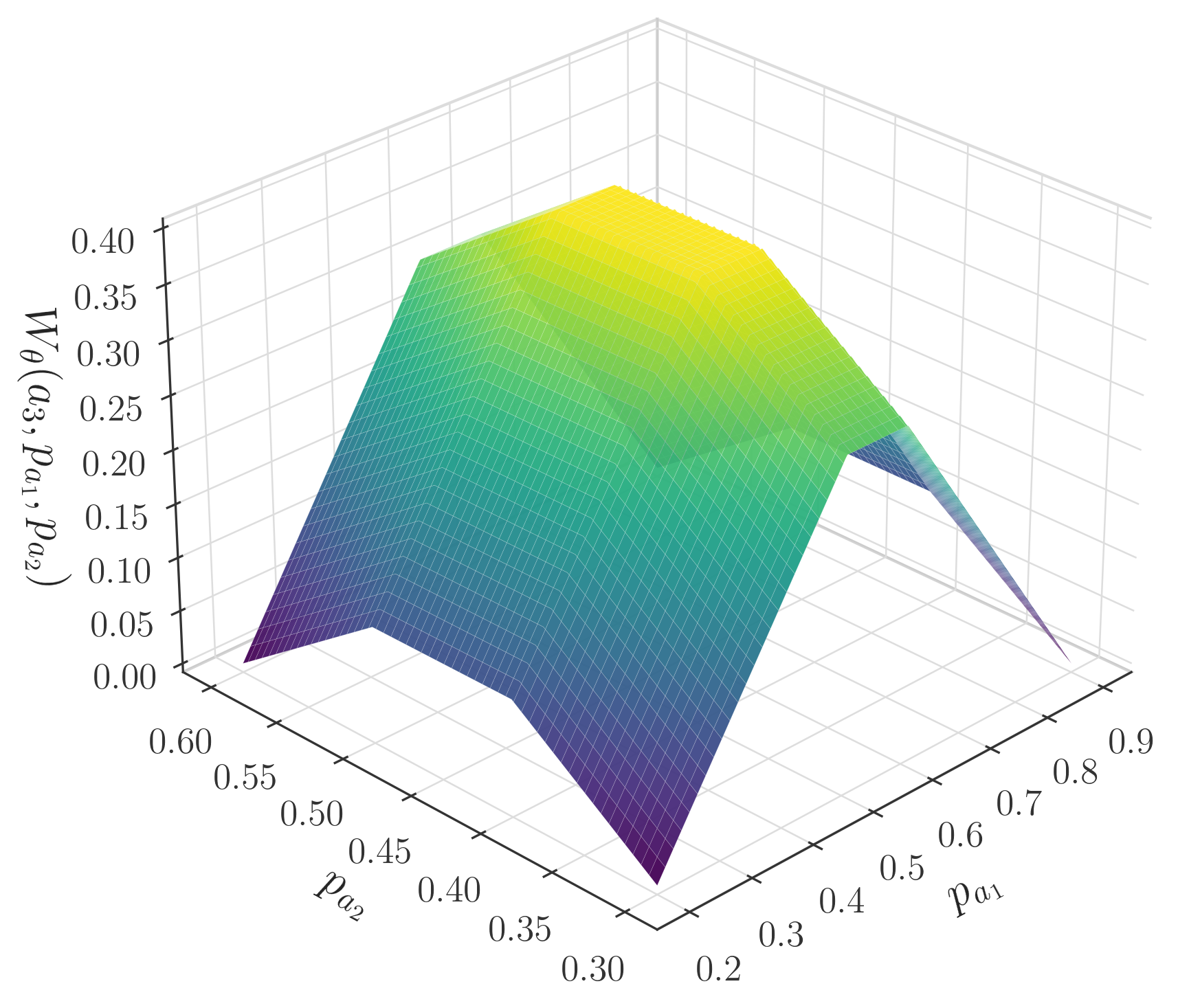}
  \caption{Experiment $\mathcal{a}_3= \{\tilde{\mathcal{a}}_1, \tilde{\mathcal{a}}_2\}$.}
  \label{fig:a3_W}
\end{subfigure}
\caption{Example~\ref{ex:potency}. The realized bound width $W_\theta(\mathcal{a}, p_\mathcal{a}):= U_\theta(\mathcal{a},p_\mathcal{a})-L_\theta(\mathcal{a},p_\mathcal{a})$ is shown as a function of $p_\mathcal{a}$.}
\label{fig:a_W}
\end{figure}

\begin{example}[Cartesian product decomposition]
\label{ex:cartesian}
Consider the chain $A \to B \to C \to D$ with confounders $R_1$ on $\{A,B\}$ and $R_2$ on $\{C,D\}$ (Figure~\ref{fig:two_district_chain}). Structural equations: $A = f_A(R_1)$, $B = f_B(A, R_1)$, $C = f_C(B, R_2)$, $D = f_D(C, R_2)$.
The districts are $\bs{D}_1=\{R_1,A, B\}$ and $\bs{D}_2=\{R_2, C, D\}$ with response-types $\bs{R}_{\bs{D}_1}=\{R_1\}$ and $\bs{R}_{\bs{D}_2}=\{R_2\}$.
\begin{figure}[H]
\centering
\begin{tikzpicture}[->, node distance=2cm, thick]
  \node[draw, circle] (A) {\( A \)};
  \node[draw, circle, right of=A] (B) {\( B \)};
  \node[draw, circle, right of=B] (C) {\( C \)};
  \node[draw, circle, right of=C] (D) {\( D \)};
  \node[draw, circle, dashed, above of=A] (R1) {\( R_1 \)};
  \node[draw, circle, dashed, above of=C] (R2) {\( R_2 \)};
  \draw (A)--(B); \draw (B)--(C); \draw (C)--(D);
  \draw (R1)--(A); \draw (R1)--(B); \draw (R2)--(C); \draw (R2)--(D);
\end{tikzpicture}
\caption{Chain graph with two districts and confounders $R_1$, $R_2$.}
\label{fig:two_district_chain}
\end{figure}
We label response-type values by the function they encode.
For $A$ (no parents), $a_0$ and $a_1$ denote the constants
$A=0$ and $A=1$. For $B$ with parent $A$, the four functions
$\{0,1\} \to \{0,1\}$ are labeled following the response-type
terminology of the literature \citep{AngristResType, GreenlandResTypes}:
\begin{table}[H]
\centering
\begin{tabular}{lcccc}
\toprule
 & $b_{\mathrm{nev}}$ & $b_{\mathrm{com}}$ & $b_{\mathrm{def}}$ & $b_{\mathrm{alw}}$ \\
 & (never-taker) & (complier) & (defier) & (always-taker) \\
\midrule
$f_B(A{=}0,\,\cdot)$ & $0$ & $0$ & $1$ & $1$ \\
$f_B(A{=}1,\,\cdot)$ & $0$ & $1$ & $0$ & $1$ \\
\bottomrule
\end{tabular}
\end{table}
The behaviors of $C$ (w.r.t.\ parent $B$) and $D$ (w.r.t.\ parent $C$) are labeled analogously.
Fix $\bs{v} = (A{=}1, B{=}0, C{=}1, D{=}0)$:
\begin{align*}
h_1(\bs{v}) &= \big\{ r_1 : f_A(r_1) = 1 \;\land\; f_B(1, r_1) = 0 \big\} = \big\{ (a_1, b_{nev}),\; (a_1, b_{def}) \big\}, \\
h_2(\bs{v}) &= \big\{ r_2 : f_C(0, r_2) = 1 \;\land\; f_D(1, r_2) = 0 \big\}
= \big\{ (c_{def}, d_{nev}),\; (c_{def}, d_{def}),\; (c_{alw}, d_{nev}),\; (c_{alw}, d_{def}) \big\}.
\end{align*}
By Lemma~\ref{lemma:cartesian}, $h(\bs{v}) = h_1(\bs{v}) \times h_2(\bs{v})$ with $|h(\bs{v})| = 2 \times 4 = 8$. For verification: $\bs{r}=(a_1, b_{nev}, c_{def}, d_{nev})$ yields $f_A(a_1)=1$, $f_B(1,b_{nev})=0$, $f_C(0,c_{def})=1$, $f_D(1,d_{nev})=0$ --- consistent with $\bs{v}$ and therefore $\bs{r} \in h(\bs{v})$. 
\end{example}

\begin{example}[Step-by-step trace of Algorithm~\ref{alg:all_paths_intercepted}]
\label{ex:pruning}
We run \textsc{AllPathsIntercepted} on Figure~\ref{fig:complexDAG} with $\bs{R}_\theta^* = \{R_1, R_2\}$, $\bs{W}=\{D\}$, $\bs{Z}=\{C\}$.
Removing $\bs{Z}$ yields the graph in Figure~\ref{fig:complexDAG_minusC} (edges to/from $C$ deleted).

\begin{figure}[H]
\centering
\begin{tikzpicture}[->, node distance=1.8cm, thick]
  \node[draw, circle] (A) {\( A \)};
  \node[draw, circle, right of=A] (B) {\( B \)};
  \node[draw, circle, right of=B] (X) {\( X \)};
  \node[draw, circle, right of=X] (M) {\( M \)};
  \node[draw, circle, right of=M] (Y) {\( Y \)};
  \node[draw, circle, below of=X, yshift=-0.3cm] (C) {\( C \)};
  \node[draw, circle, right of=C] (D) {\( D \)};
  \node[draw, circle, dashed, above of=B, yshift=0.3cm] (R3) {\( R_3 \)};
  \node[draw, circle, dashed, above of=X, xshift=0.9cm, yshift=0.3cm] (R1) {\( R_1 \)};
  \node[draw, circle, dashed, above of=M, xshift=0.9cm, yshift=0.3cm] (R2) {\( R_2 \)};
  \node[draw, circle, dashed, below of=B, xshift=0.0cm, yshift=-0.3cm] (R4) {\( R_4 \)};
  \draw (A) -- (B); \draw (B) -- (X); \draw (X) -- (M); \draw (M) -- (Y);
  \draw (D) -- (M);
  \draw[dashed] (R3) -- (A); \draw[dashed] (R3) -- (B);
  \draw[dashed] (R1) -- (X); \draw[dashed] (R1) -- (M);
  \draw[dashed] (R2) -- (M); \draw[dashed] (R2) -- (Y);
  \draw[dashed] (R4) -- (B);
\end{tikzpicture}
\caption{Figure~\ref{fig:complexDAG} with all edges to/from $C$ removed (i.e., after deleting $\bs{Z}=\{C\}$).}
\label{fig:complexDAG_minusC}
\end{figure}

After initialization: $\texttt{visited} = \{R_1, R_2\}$, $\texttt{queue} = (R_1, R_2)$.

\textit{Iteration 1}: $\texttt{current} = R_1$. Not in $\bs{W}$. Children in $\bs{E}'$: $X$, $M$ (unvisited). Update $\texttt{visited} = \{R_1, R_2, X, M\}$, $\texttt{queue} = (R_2, X, M)$.

\textit{Iteration 2}: $\texttt{current} = R_2$. Unvisited child: $Y$. Update $\texttt{visited} = \{R_1, R_2, X, M, Y\}$, $\texttt{queue} = (X, M, Y)$.

\textit{Iterations 3--5}: $X$ has no unvisited children; $M$'s child $Y$ already visited; $Y$ has no children.

\noindent Return \texttt{True}. No path from $\bs{R}_\theta^*$ to $D$ exists in $\bs{E}'$. By Corollary~\ref{th-experiments_with_intercepted_paths_useless}, this intervention is useless.
\end{example}

\section{Algorithmic Details}
\label{app:alg-details}
\subsection{Constructing the Polynomial Program}
\label{app:program}

\textbf{Canonicalization.} We assume throughout that the input ADMG is 
canonicalized---no disturbance has a parent in $\mathcal{G}$, and no 
disturbance influences a strict subset of another's children. This is 
without loss of generality for the full data law 
\citep[Proposition~1]{duarte2023automated}; see Appendix~B.1 of that 
paper for the canonicalization procedure.

\textbf{Constructing $\mathcal{P}_\emptyset$.} The base program is 
obtained by invoking Algorithm~2 of \citet{duarte2023automated} with 
inputs $(\mathcal{G}, P(\bs{V}), \emptyset, supp(\bs{V}), \theta)$.

\textbf{Constructing $\mathcal{P}_{\mathcal{a}}$.} Given the base program, 
each experimental statement is polynomialized via the same procedure 
applied to the equality $\tilde{\mathcal{a}} = p_{\tilde{\mathcal{a}}}$:

\begin{algorithm}[H]
\caption{Formulating $\mathcal{P}_{\mathcal{a}}$}
\label{alg:formulation}
\begin{algorithmic}[1]
\Require Base Program $\mathcal{P}_\emptyset = (\mathcal{T}, \mathcal{C}_\emptyset)$, Experiments $\mathcal{a} \subseteq \mathcal{A}$
\Ensure Program $\mathcal{P}_{\mathcal{a}} = (\mathcal{T}, \mathcal{C}_{\mathcal{a}})$
\State $\mathcal{C}_{\mathcal{a}} \gets \mathcal{C}_\emptyset$ \Comment{Initialize with base constraints}
\For{$\tilde{\mathcal{a}} \in \mathcal{a}$}
    \State $c_{\tilde{\mathcal{a}}} \gets \text{duarte.polynomialize}(\tilde{\mathcal{a}} = p_{\tilde{\mathcal{a}}})$
    \State $\mathcal{C}_{\mathcal{a}} \gets \mathcal{C}_{\mathcal{a}} \cup \{c_{\tilde{\mathcal{a}}}\}$
\EndFor
\State \Return $\mathcal{P}_{\mathcal{a}} = (\mathcal{T}, \mathcal{C}_{\mathcal{a}})$
\end{algorithmic}
\end{algorithm}

\subsection{Evaluating Potency}
\label{subsec:evaluating-potency}
For a fixed experiment set $\mathcal{a}$, potency evaluation requires the
worst-case post-experimental width
\[
\mathcal{W}_\theta(\mathcal{a})
=
\max_{p_{\mathcal{a}}}
\left[
U_\theta(\mathcal{a},p_{\mathcal{a}})
-
L_\theta(\mathcal{a},p_{\mathcal{a}})
\right].
\]
Unfolding the inner bounds as polynomial programs over independent
optimizers $P^U$ and $P^L$, both required to satisfy the observational
constraints and to reproduce the experimental outcome $p_{\mathcal{a}}$,
\[
\mathcal{W}_\theta(\mathcal{a})
=
\max_{p_{\mathcal{a}}}\;
\max_{\substack{P^U \in \mathcal{C}_\emptyset \\ f(P^U)=p_{\mathcal{a}}}}\;
\max_{\substack{P^L \in \mathcal{C}_\emptyset \\ f(P^L)=p_{\mathcal{a}}}}
\bigl[\mathcal{T}(P^U) - \mathcal{T}(P^L)\bigr].
\]
The variable $p_{\mathcal{a}}$ appears only as the common value of
$f(P^U)$ and $f(P^L)$. We can therefore drop it as a decision variable
and replace its two binding constraints with the single equality
$f(P^U) = f(P^L)$. Whatever value the two optimizers agree on \emph{is}
the experimental outcome; no separate enumeration of $p_{\mathcal{a}}$
is needed.

This also resolves the question of feasibility for $p_{\mathcal{a}}$.
In the outer-loop view one would first compute bounds on
$p_{\mathcal{a}}$ under $\mathcal{C}_\emptyset$ and grid over them; here
those bounds are enforced implicitly, since the existence of $P^U, P^L
\in \mathcal{C}_\emptyset$ producing a common value $p$ is exactly the
definition of $p$ being feasible. See Algorithm~\ref{alg:evaluate-potency}.

\begin{algorithm}[H]
\caption{\textsc{EvaluatePotency}}
\label{alg:evaluate-potency}
\begin{algorithmic}[1]
\Require Base program $\mathcal{P}_\emptyset=(\mathcal{T},\mathcal{C}_\emptyset)$, experiment set $\mathcal{a}$
\Ensure Potency $pot_\theta(\mathcal{a})$
\State $L_\theta \gets \min \mathcal{T}$ subject to $\mathcal{C}_\emptyset$
\State $U_\theta \gets \max \mathcal{T}$ subject to $\mathcal{C}_\emptyset$
\State $\mathcal{W}_\theta(\emptyset) \gets U_\theta-L_\theta$
\State Create two copies of the response-type probabilities:
\[
\{P^U(\bs{r}) : \bs{r}\in supp(\bs{R})\},
\qquad
\{P^L(\bs{r}) : \bs{r}\in supp(\bs{R})\}.
\]
\State Let $\mathcal{C}_\emptyset^U$ and $\mathcal{C}_\emptyset^L$ denote
copies of $\mathcal{C}_\emptyset$ using $P^U(\bs r)$ and $P^L(\bs r)$,
respectively.
\State Form the coupled constraint set
\[
\mathcal{C}^{\mathrm{cpl}}_{\mathcal{a}}
=
\mathcal{C}_\emptyset^U
\cup
\mathcal{C}_\emptyset^L
\cup
\left\{
f_{\tilde{\mathcal{a}}}(\bs{R}^U)
=
f_{\tilde{\mathcal{a}}}(\bs{R}^L)
:
\tilde{\mathcal{a}}\in\mathcal{a}
\right\}.
\]
\State Compute
\[
\mathcal{W}_\theta(\mathcal{a})
\gets
\max
\left[
\mathcal{T}(\bs{R}^U)-\mathcal{T}(\bs{R}^L)
\right]
\quad
\text{s.t.}
\quad
\mathcal{C}^{\mathrm{cpl}}_{\mathcal{a}}.
\]
\State \Return $\mathcal{W}_\theta(\emptyset)-\mathcal{W}_\theta(\mathcal{a})$
\end{algorithmic}
\end{algorithm}
The formulation is exact at the mathematical level; numerical solutions
are subject to the tolerances of the polynomial optimizer.

\subsection{AllPathsIntercepted Algorithm}\label{app:algorithms}

\begin{algorithm}[h]
\caption{AllPathsIntercepted}
\label{alg:all_paths_intercepted}
\begin{algorithmic}
\Require Graph $\mathcal{G} = (\bs{V}, \bs{R}, \bs{E})$, $\bs{R}_\theta^*$, Outcomes $\bs{W}$, Interventions $\bs{Z}$
\Ensure \texttt{True} iff every directed path $\bs{R}_\theta^* \to \bs{W}$ passes through $\bs{Z}$
\State $\bs{E}' \gets \{ (u,v) \in \bs{E} : u \notin \bs{Z} \land v \notin \bs{Z} \}$ \Comment{Remove edges involving $\bs{Z}$}
\State $\texttt{visited} \gets \bs{R}_\theta^*$; $\quad\texttt{queue} \gets \texttt{list}(\bs{R}_\theta^*)$
\While{$\texttt{queue} \neq \emptyset$}
    \State $\texttt{current} \gets \texttt{queue.pop}()$
    \If{$\texttt{current} \in \bs{W}$} \Return \texttt{False} \Comment{Unintercepted path found} \EndIf
    \For{$c \in \texttt{children}(\texttt{current}, \bs{E}')$ with $c \notin \texttt{visited}$}
        \State $\texttt{visited} \gets \texttt{visited} \cup \{c\}$; $\quad\texttt{queue.append}(c)$
    \EndFor
\EndWhile
\State \Return \texttt{True} \Comment{All paths intercepted}
\end{algorithmic}
\end{algorithm}

\textbf{Runtime of Algorithm~\ref{alg:all_paths_intercepted}}
\label{proof:runtime}

We analyze each component of Algorithm~\ref{alg:all_paths_intercepted}:

\textit{Initialization:}
Constructing $\bs{E}'$ requires iterating through all edges and performing two membership checks per edge, yielding $\mathcal{O}(|\bs{E}|)$. Initializing \texttt{visited} and \texttt{queue} with $\bs{R}_\theta^*$ takes $\mathcal{O}(|\bs{R}^*_\theta|) \leq \mathcal{O}(|\bs{V}|)$ (this inequality holds in canonical DAGs).

\textit{BFS Loop:}
Each node in $\bs{V} \setminus \bs{Z}$ is added to \texttt{visited} at most once and consequently to \texttt{queue} at most once. Therefore, the while-loop executes at most $|\bs{V} \setminus \bs{Z}| \leq |\bs{V}|$ iterations, contributing $\mathcal{O}(|\bs{V}|)$ for dequeue operations and membership checks.

For the inner loop, each edge $(u,v) \in \bs{E}'$ is examined at most once---when node $u$ is dequeued and we iterate over its children. Since every edge is examined at most once across the entire execution, the total work by the inner loop is $\mathcal{O}(|\bs{E}'|) \subseteq \mathcal{O}(|\bs{E}|)$.

\textit{Total:}
$$\underbrace{\mathcal{O}(|\bs{V}| + |\bs{E}|)}_{\text{init.}} + \underbrace{\mathcal{O}(|\bs{V}|)}_{\text{while-loop}} + \underbrace{\mathcal{O}(|\bs{E}|)}_{\text{edge exams.}} = \mathcal{O}(|\bs{V}| + |\bs{E}|). \qedhere$$

\subsection{Solving the Max-Potency Problem}
\label{subsec:solving-max-potency}

We use \textsc{EvaluatePotency} as a subroutine to solve the
budget-constrained max-potency problem exactly. Since the problem is
NP-hard (Appendix~\ref{app:np-hardness}), the algorithm is exponential
in the worst case.

Algorithm~\ref{alg:get_useless_experiments} certifies a set
$\mathcal{A}^\dagger$ of individually useless experiments. Let
$\mathcal{A}^\dagger_{\mathrm{ID}} \subseteq \mathcal{A}^\dagger$
denote those pruned by the ID check and
$\mathcal{A}^\dagger_{\mathrm{Int}} = \mathcal{A}^\dagger \setminus
\mathcal{A}^\dagger_{\mathrm{ID}}$ those pruned by interception.
By Corollary~\ref{thm-id_inertness},
$\mathcal{A}^\dagger_{\mathrm{ID}}$ can be removed from the search
entirely. The algorithm then enumerates all subsets of
$\mathcal{A} \setminus \mathcal{A}^\dagger_{\mathrm{ID}}$, skipping
singletons in $\mathcal{A}^\dagger_{\mathrm{Int}}$.

\begin{algorithm}[H]
\caption{\textsc{SolveMaxPotency}}
\label{alg:solve-max-potency}
\begin{algorithmic}[1]
\Require Graph $\mathcal G$, query $\theta$, base program
$\mathcal P_\emptyset=(\mathcal T,\mathcal C_\emptyset)$, candidate
experiments $\mathcal A$, cost function
$c:2^{\mathcal A}\to\mathbb R_+$, budget $B$
\Ensure Optimal experiment set $\mathcal a^\star$ and potency
value $v^\star$
\State $\mathcal A^\dagger
\gets
\textsc{GetUselessExperiments}(\mathcal G,\theta,\mathcal A)$
\State Partition $\mathcal{A}^\dagger$ into
$\mathcal{A}^\dagger_{\mathrm{ID}}$ and
$\mathcal{A}^\dagger_{\mathrm{Int}}$
\State $\mathcal a^\star \gets \emptyset$;\;
$v^\star \gets 0$
\For{$\mathcal a\subseteq \mathcal{A} \setminus
\mathcal{A}^\dagger_{\mathrm{ID}}$}
    \If{$|\mathcal a| = 1$ \textbf{and}
    $\mathcal{a} \subseteq \mathcal{A}^\dagger_{\mathrm{Int}}$}
        \State \textbf{continue}
    \EndIf
    \If{$c(\mathcal a)>B$}
        \State \textbf{continue}
    \EndIf
    \State $v \gets
    \textsc{EvaluatePotency}(\mathcal P_\emptyset,\mathcal a)$
    \If{$v>v^\star$}
        \State $v^\star \gets v$
        \State $\mathcal a^\star \gets \mathcal a$
    \EndIf
\EndFor
\State \Return $\mathcal a^\star, v^\star$
\end{algorithmic}
\end{algorithm}

The algorithm is exact provided that \textsc{EvaluatePotency}
(Algorithm~\ref{alg:evaluate-potency}) globally solves the coupled
polynomial program. Its runtime is
$\mathcal{O}(2^{|\mathcal{A}| - |\mathcal{A}^\dagger_{\mathrm{ID}}|}
\cdot T_{\mathrm{pot}})$,
where $T_{\mathrm{pot}}$ denotes the cost of one potency evaluation.
Corollary~\ref{thm-id_inertness} reduces the exponent from $|\mathcal{A}|$ to
$|\mathcal{A}| - |\mathcal{A}^\dagger_{\mathrm{ID}}|$.

\section{Omitted Proofs}
\label{sec:proofs}

\paragraph{Feasible sets.}
We define two types of feasible sets used in the proofs below. The \emph{observational feasible set} $\mathcal{F}_\emptyset := \{ \bs{r} \in supp(\bs{R}): \text{all observational constraints from } P(\bs{V}) \text{ are satisfied}\}$ is the feasible set of the base program $\mathcal{P}_\emptyset$ as constructed by \citet{duarte2023automated}. For an experiment $\mathcal{a}$ with observed outcome $p_\mathcal{a}$, the \emph{experimental feasible set} is $\mathcal{F}_\mathcal{a}(p_\mathcal{a}) := \{ \bs{r} \in supp(\bs{R}): f_\mathcal{a}(\bs{r}) = p_\mathcal{a} \}$. The combined feasible set is $\mathcal{F}_\emptyset \cap \mathcal{F}_\mathcal{a}(p_\mathcal{a})$. 

\subsection{Proof of Lemma~\ref{lemma-pot_greater_zero}}
\label{proof:pot_greater_zero}

\begin{proof}
Let $p_\mathcal{a}\in\prod_{\tilde{\mathcal{a}} \in \mathcal{a}}[l_{p_{\tilde{\mathcal{a}}}}, u_{p_{\tilde{\mathcal{a}}}}]$. The combined feasible set $\mathcal{F}_\emptyset \cap \mathcal{F}_\mathcal{a}(p_\mathcal{a}) \subseteq \mathcal{F}_\emptyset$ by definition of intersection. From optimization theory, restricting to a smaller feasible set can only decrease the maximum and increase the minimum. Hence $U_\theta(\mathcal{a}, p_\mathcal{a}) \le U_\theta$ and $L_\theta \le L_\theta(\mathcal{a}, p_\mathcal{a})$.
It follows $\forall p_\mathcal{a}$:
$$
0 \leq U_\theta- U_\theta(\mathcal{a}, p_\mathcal{a}) \;\land\; L_\theta - L_\theta(\mathcal{a}, p_\mathcal{a}) \leq 0 \implies L_\theta - L_\theta(\mathcal{a}, p_\mathcal{a}) \leq U_\theta- U_\theta(\mathcal{a}, p_\mathcal{a})
$$$$\implies U_\theta(\mathcal{a}, p_\mathcal{a}) - L_\theta(\mathcal{a}, p_\mathcal{a}) \leq U_\theta- L_\theta =\mathcal{W}_\theta
$$
Since this holds for all $p_\mathcal{a}$, we have $\mathcal{W}_\theta(\mathcal{a}) = \underset{\substack{p_{\mathcal{a}}}}{\max} \bigl(U_\theta(\mathcal{a}, p_{\mathcal{a}}) - L_\theta(\mathcal{a}, p_{\mathcal{a}})\bigr) \leq \mathcal{W}_\theta$, and hence $pot_\theta(\mathcal{a}) \geq 0$.
\end{proof}

\subsection{Proof of Lemma~\ref{lemma_pot0_implies_boundsEqual}}
\label{proof:pot0_implies_boundsEqual}

\begin{proof}
\textbf{($\Rightarrow$)} Unpacking the definition yields
    $$pot_\theta(\mathcal{a}) =0 \implies \mathcal{W}_\theta(\mathcal{a}) = \mathcal{W}_\theta \implies \underset{\substack{p_{\mathcal{a}}\in\prod_{\tilde{\mathcal{a}} \in \mathcal{a}}[l_{p_{\tilde{\mathcal{a}}}}, u_{p_{\tilde{\mathcal{a}}}}]}}{\max} \bigl(U_\theta(\mathcal{a}, p_{\mathcal{a}}) - L_\theta(\mathcal{a}, p_{\mathcal{a}})\bigr) = U_\theta- L_\theta.$$
    Let $p_\mathcal{a}^\dagger = \underset{\substack{p_{\mathcal{a}}\in\prod_{\tilde{\mathcal{a}} \in \mathcal{a}}[l_{p_{\tilde{\mathcal{a}}}}, u_{p_{\tilde{\mathcal{a}}}}]}}{\operatorname{arg\,max}} \ U_\theta(\mathcal{a}, p_\mathcal{a}) - L_\theta(\mathcal{a}, p_\mathcal{a})$:
    $$U_\theta(\mathcal{a}, p_\mathcal{a}^\dagger) - L_\theta(\mathcal{a}, p_\mathcal{a}^\dagger) = U_\theta- L_\theta \implies U_\theta- U_\theta(\mathcal{a}, p_\mathcal{a}^\dagger) = L_\theta - L_\theta(\mathcal{a}, p_\mathcal{a}^\dagger).$$
    As we saw in the proof of Lemma~\ref{lemma-pot_greater_zero}, $L_\theta \leq L_\theta(\mathcal{a}, p_\mathcal{a}^\dagger)$ and $U_\theta(\mathcal{a}, p_\mathcal{a}^\dagger) \leq U_\theta$, which implies that\footnote{Since $U_\theta- U_\theta(\mathcal{a}, p_\mathcal{a}^\dagger) \geq 0$ and $L_\theta - L_\theta(\mathcal{a}, p_\mathcal{a}^\dagger) \leq 0$, their equality forces both to be zero.}
    $$U_\theta- U_\theta(\mathcal{a}, p_\mathcal{a}^\dagger) =0 \;\land\; L_\theta - L_\theta(\mathcal{a}, p_\mathcal{a}^\dagger)=0 \implies U_\theta(\mathcal{a}, p_\mathcal{a}^\dagger)= U_\theta\;\land\; L_\theta(\mathcal{a}, p_\mathcal{a}^\dagger) = L_\theta.$$

\textbf{($\Leftarrow$)} Suppose there exists $p_\mathcal{a}^\dagger \in \prod_{\tilde{\mathcal{a}} \in \mathcal{a}}[l_{p_{\tilde{\mathcal{a}}}}, u_{p_{\tilde{\mathcal{a}}}}]$ with $U_\theta(\mathcal{a}, p_\mathcal{a}^\dagger) = U_\theta$ and $L_\theta(\mathcal{a}, p_\mathcal{a}^\dagger) = L_\theta$. Then $$U_\theta(\mathcal{a},p_\mathcal{a}^\dagger) - L_\theta(\mathcal{a}, p_\mathcal{a}^\dagger) = U_\theta- L_\theta = \mathcal{W}_\theta.$$ Since $\mathcal{W}_\theta(\mathcal{a})$ is the maximum over all $p_\mathcal{a}$, it must hold that $\mathcal{W}_\theta(\mathcal{a}) \geq \mathcal{W}_\theta$. By Lemma~\ref{lemma-pot_greater_zero}, $\mathcal{W}_\theta(\mathcal{a}) \leq \mathcal{W}_\theta$, hence $\mathcal{W}_\theta(\mathcal{a}) = \mathcal{W}_\theta$ and $pot_\theta(\mathcal{a}) = 0$.
\end{proof}

\subsection{Proof of Lemma~\ref{lemma:cartesian} (Cartesian Product Decomposition)}
\label{proof:cartesian}

\begin{proof}
We start by applying Definition~\ref{def:compatible-restypes} to $h(\bs{v})$, simplified for the case of $|\frak{C}|=1$:

$$
h(\bs{v})=\{\bs{r} : \forall V \in \bs{V}, F_V(\bs{r})=v\}.
$$
Since $\bs{X}=\emptyset$ (no interventions), $F_V(\bs{r})$ reduces to the structural equation with all parent values recursively fixed by $\bs{v}$:
$$F_V(\bs{r})=f_V\!\left(\bigwedge_{P \in \bs{pa}(V)} \underbrace{f(f(\dots f(\bs{r}),\bs{r}),\bs{r})}_{\text{equals elements of }\bs{v}}, \bs{r}\right).$$
Using Definition~\ref{def:district_compatible_restype} to replace the recursive evaluation:
$$
h(\bs{v})=\{\bs{r} : \forall V \in \bs{V},\; f_V (\bs{pa}_{\bs{v}}, \bs{r})=v \}.$$

Rewriting $\forall V \in \bs{V}$ as $\forall \bs{D}_k \in \bs{D}, \forall V \in \bs{D}_k \cap \bs{V}$ and partitioning $\bs{r}$ into per-district vectors $\bs{r}_{1},\dots,\bs{r}_{|\bs{D}|}$:
$$
h(\bs{v})=\{(\bs{r}_1, \dots, \bs{r}_{|\bs{D}|}) : \forall \bs{D}_k \in \bs{D}, \forall V \in \bs{D}_k \cap \bs{V},\; f_V (\bs{pa}_{\bs{v}}, \bs{r})=v \}.
$$
Since $f_V(\bs{pa}_{\bs{v}}, \bs{r})=f_V(\bs{pa}_{\bs{v}}, \bs{r}_{k})$ whenever $V \in \bs{D}_k$:
$$h(\bs{v})=\{(\bs{r}_1, \dots, \bs{r}_{|\bs{D}|}) : \forall \bs{D}_k \in \bs{D}, \forall V \in \bs{D}_k \cap \bs{V},\; f_V(\bs{pa}_{\bs{v}}, \bs{r}_{k}) =v \}.$$
The predicate is a conjunction where each conjunct depends only on the response-type configuration $\bs{r}_{k}$ of one district $\bs{D}_k$. By Lemma~\ref{lemma:cartesian-filter} (Appendix~\ref{sec:helper_lemmas}), we can factorize this set:
$$h(\bs{v})=\prod_{\bs{D}_k\in \bs{D}} \{\bs{r}_k : \forall V \in \bs{V}\cap \bs{D}_k,\; f_V(\bs{pa}_{\bs{v}}(V), \bs{r}_k)=v \} = \prod_{\bs{D}_k\in \bs{D}} h_k(\bs{v}).$$
\end{proof}

\subsection{Proof of Lemma~\ref{lemma:district_cfactor}}
\label{proof:district_cfactor}

\begin{proof}
\textbf{Single district.}
By Eq.~(26) from \cite{tianPearlQfactor}, the c-factor for district $\bs{D}_k$ is
$$Q_k(\bs{v})=\sum_{\bs{u}_k \in supp(\bs{U}_k)} \prod_{V_i \in \bs{V}_k} P(v_i \mid \bs{pa}_{\bs{v}}(V_i),\bs{u}_k)\, P(\bs{u}_k),$$
where $\bs{U}_k$ denotes the latent confounders in $\bs{D}_k$. Replacing $\bs{U}_k$ by our discrete $\bs{R}_k$:
$$Q_k(\bs{v})=\sum_{\bs{r}_k} \prod_{V_i \in \bs{V}_k} P(v_i \mid \bs{pa}_{\bs{v}}(V_i),\bs{r}_k)\, P(\bs{r}_k).$$
Each realization $\bs{r}_k$ deterministically fixes the structural equations, so the conditional probability becomes an indicator:
$$Q_k(\bs{v})=\sum_{\bs{r}_k} \prod_{V_i \in \bs{V}_k} \bs{1}\!\big[f_{V_i}(\bs{pa}_{\bs{v}}(V_i),\, \bs{r}_k) = v_i\big]\, P(\bs{r}_k).$$
The product of indicators equals 1 if and only if $\bs{r}_k \in h_k(\bs{v})$, yielding $Q_{k}(\bs{v}) = \sum_{\bs{r}_{k} \in h_{k}(\bs{v})} P(\bs{r}_k)$.

\textbf{Set of districts.}
We show the result for $|\bs{D}'|=2$; the general case follows by induction. Let $\bs{D}' = \{\bs{D}_1, \bs{D}_2\}$. Applying the single-district result:
$$\prod_{\bs{D}_k \in \bs{D}'} Q_k(\bs{v}) = \left(\sum_{r_{1} \in h_{1}(\bs{v})} P(r_{1})\right)\left(\sum_{r_{2} \in h_{2}(\bs{v})} P(r_{2})\right).$$
Expanding the product of sums:
$$= \sum_{r_{1} \in h_{1}(\bs{v})} \sum_{r_{2} \in h_{2}(\bs{v})} P(r_{1})\, P(r_{2}).$$
Since response-types in distinct districts are independent, $P(r_{k_1})\, P(r_{k_2}) = P(\bs{r}_{\bs{D}'})$. Collapsing the double sum into a sum over the Cartesian product\footnote{Formally justified by Lemma~\ref{lemma:cartesian-double-sum}.} and applying Lemma~\ref{lemma:cartesian}:
$$= \sum_{\bs{r}_{\bs{D}'} \in h_{1}(\bs{v}) \times h_{2}(\bs{v})} P(\bs{r}_{\bs{D}'}) = \sum_{\bs{r}_{\bs{D}'} \in h_{\bs{D}'}(\bs{v})} P(\bs{r}_{\bs{D}'}). \qedhere$$
\end{proof}

\subsection{Proof of Lemma~\ref{th-all_constr_on_Rtheta_in_Dtheta_c-factor}}
\label{proof:constr_on_Rtheta}

\begin{proof}
We write the observational constraint for a fixed $\bs{v} \in supp(\bs{V})$:
$$P(\bs{v})= \sum_{\bs{r} \in h(\bs{v})}P(\bs{r}).$$
Let $\bs{D}_{\bar\theta} = \bs{D} \setminus \bs{D}^*_\theta$ denote all districts not in $\bs{D}^*_\theta$ and $\bs{R}_{\bar\theta} = \bs{D}_{\bar\theta} \cap \bs{R}$ their response types. Since response-types in distinct districts are independent, $P(\bs{r}) = P(\bs{r}_{\theta}^*)\, P(\bs{r}_{\bar\theta})$:
$$P(\bs{v})= \sum_{\bs{r} \in h(\bs{v})} P(\bs{r}_{\theta}^*)\, P(\bs{r}_{\bar\theta}).$$
By Lemma~\ref{lemma:cartesian}, $h(\bs{v}) = \prod_{\bs{D}_k \in \bs{D}} h_k(\bs{v})$, which we partition as $h_{\bs{D}_\theta^*}(\bs{v}) \times h_{\bs{D}_{\bar\theta}}(\bs{v})$. Writing as a double summation and pulling out $P(\bs{r}_\theta^*)$ (independent of $\bs{r}_{\bar\theta}$):
$$P(\bs{v}) = \sum_{\bs{r}_\theta^* \in h_{\bs{D}_\theta^*}(\bs{v})} P(\bs{r}_\theta^*) \sum_{\bs{r}_{\bar\theta} \in h_{\bs{D}_{\bar\theta}}(\bs{v})}  P(\bs{r}_{\bar\theta}).$$
Since $\bs{D}^*_\theta \cup \bs{D}_{\bar\theta} = \bs{D}$, the c-factor decomposition gives $P(\bs{v}) = \prod_{\bs{D}_k \in \bs{D}^*_\theta}Q_k(\bs{v}) \cdot \prod_{\bs{D}_l \in \bs{D}_{\bar\theta}}Q_l(\bs{v})$:
$$\prod_{\bs{D}_k \in \bs{D}^*_\theta}Q_k(\bs{v}) \prod_{\bs{D}_l \in \bs{D}_{\bar\theta}}Q_l(\bs{v}) =
\sum_{\bs{r}_\theta^* \in h_{\bs{D}_\theta^*}(\bs{v})} P(\bs{r}_\theta^*) \sum_{\bs{r}_{\bar\theta} \in h_{\bs{D}_{\bar\theta}}(\bs{v})}  P(\bs{r}_{\bar\theta}).$$
By Lemma~\ref{lemma:district_cfactor}, $\prod_{\bs{D}_l \in \bs{D}_{\bar\theta}}Q_l(\bs{v})= \sum_{\bs{r}_{\bar\theta} \in h_{\bs{D}_{\bar\theta}}(\bs{v})} P(\bs{r}_{\bar\theta})$. Dividing both sides by $\prod_{\bs{D}_l \in \bs{D}_{\bar\theta}}Q_l(\bs{v}) > 0$ yields the desired identity:
$$\prod_{\bs{D}_k \in \bs{D}^*_\theta}Q_k(\bs{v}) = \sum_{\bs{r}_\theta^* \in h_{\bs{D}_\theta^*}(\bs{v})} P(\bs{r}_\theta^*).$$
\end{proof}


\begin{lemma}
\label{th-P_empty_optimizes_over_R_theta}
The base program $\mathcal{P}_\emptyset$ optimizes only over $\bs{R}_\theta^*$.
\end{lemma}
\begin{proof}
First, by construction of $\bs{R}_\theta^*$, it covers all response-types necessary for expressing $\theta$. Therefore the target polynomial $\mathcal{T}$ is only a function of $\bs{R}_\theta^*$.
Second, Lemma~\ref{th-all_constr_on_Rtheta_in_Dtheta_c-factor} shows that all observational constraints depend only on $\bs{R}_\theta^*$. Hence $\mathcal{C}_\emptyset$ is only a function of $\bs{R}_\theta^*$.
Since both the target function and the constraints depend only on $\bs{R}_\theta^*$, it is sufficient to optimize over $\bs{R}_\theta^*$ in $\mathcal{P}_\emptyset$.
\end{proof}

\subsection{Proof of Theorem~\ref{th-doesnt_touch_r_theta_implies_pot0}}
\label{proof:th-doesnt_touch_r_theta_implies_pot0}
\begin{proof}
Since $\mathcal{P}_\emptyset$ optimizes over $\bs{R}_\theta^*$ only (Lemma~\ref{th-P_empty_optimizes_over_R_theta}), there is no sequence of co-occurring decision variables connecting $R^\star(\tilde{\mathcal{a}})$ to $\mathcal{T}$. By \citet{duarte2023automated}, the constraint for $\tilde{\mathcal{a}}$ can be dropped, resulting in  $\mathcal{P}_{\tilde{\mathcal{a}}} = \mathcal{P}_\emptyset$ and $pot_\theta(\tilde{\mathcal{a}}) = 0$.
\end{proof}

\subsection{Proof of Corollary~\ref{th-experiments_with_intercepted_paths_useless}}
\label{proof:intercepted_paths}

\begin{proof}
By assumption, for every $\frak{c} \in \frak{C}$, every directed path from $\bs{R}_\theta^*$ to $\bs{Y}^{(\frak{c})}$ passes through $\bs{X}^{(\frak{c})}$. By the definition of $R^\star(\tilde{\mathcal{a}})$, a response-type $R$ is in $R^\star(\tilde{\mathcal{a}})$ only if there exists some $\frak{c} \in \frak{C}$ and a directed path from $R$ to $\bs{Y}^{(\frak{c})}$ that does \emph{not} pass through $\bs{X}^{(\frak{c})}$. Since no such path exists for any $\frak{c}$, no element of $\bs{R}_\theta^*$ can be in $R^\star(\tilde{\mathcal{a}})$. Hence $R^\star(\tilde{\mathcal{a}}) \cap \bs{R}_\theta^* = \emptyset$. By Theorem~\ref{th-doesnt_touch_r_theta_implies_pot0}, we have $pot_\theta(\tilde{\mathcal{a}})=0$.
\end{proof}

\subsection{Proof of Theorem~\ref{id_implies_pot0}}
\label{proof:id_implies_pot0}

\begin{proof}
Since $\tilde{\mathcal{a}}$ is ID, it is uniquely computable from $P(\bs{V})$. Therefore, we can determine a unique value $p_{\tilde{\mathcal{a}}}^*$ for the experimental outcome, which satisfies $\forall \bs{r} \in \mathcal{F}_\emptyset: f_{\tilde{\mathcal{a}}}(\bs{r})= p_{\tilde{\mathcal{a}}}^*$. Hence $\mathcal{F}_\emptyset \subseteq \mathcal{F}_{\tilde{\mathcal{a}}}(p_{\tilde{\mathcal{a}}}^*)$, which gives:
$$\mathcal{F}_\emptyset \cap \mathcal{F}_{\tilde{\mathcal{a}}}(p_{\tilde{\mathcal{a}}}^*) = \mathcal{F}_\emptyset.$$
Since the combined feasible set coincides with $\mathcal{F}_\emptyset$, the optimization is unchanged: $\mathcal{W}_\theta(\tilde{\mathcal{a}}) = \mathcal{W}_\theta(\emptyset)$ and $pot_\theta(\tilde{\mathcal{a}})=0$.
\end{proof}

\subsection{Proof of Corollary~\ref{thm-id_inertness}}
\label{proof:thm-id_inertness}
\begin{proof}
Since $\tilde{\mathcal{a}}$ is identifiable, it has a unique feasible outcome $p_{\tilde{\mathcal{a}}}^*$ with $\mathcal{F}_\emptyset \cap \mathcal{F}_{\tilde{\mathcal{a}}}(p_{\tilde{\mathcal{a}}}^*) = \mathcal{F}_\emptyset$ (Theorem~\ref{id_implies_pot0}). For any $\mathcal{b} \subseteq \mathcal{A}$ and feasible $p_{\mathcal{b}}$:
$$
\mathcal{F}_\emptyset \cap \mathcal{F}_{\{\tilde{\mathcal{a}}\} \cup \mathcal{b}}(p_{\tilde{\mathcal{a}}}^*, p_{\mathcal{b}}) = \mathcal{F}_\emptyset \cap \mathcal{F}_{\tilde{\mathcal{a}}}(p_{\tilde{\mathcal{a}}}^*) \cap \mathcal{F}_{\mathcal{b}}(p_{\mathcal{b}}) = \mathcal{F}_\emptyset \cap \mathcal{F}_{\mathcal{b}}(p_{\mathcal{b}}).
$$
Since $p_{\tilde{\mathcal{a}}}^*$ is the only feasible value of $p_{\tilde{\mathcal{a}}}$, the worst-case width satisfies $\mathcal{W}_\theta(\{\tilde{\mathcal{a}}\} \cup \mathcal{b}) = \mathcal{W}_\theta(\mathcal{b})$, and hence $pot_\theta(\{\tilde{\mathcal{a}}\} \cup \mathcal{b}) = pot_\theta(\mathcal{b})$.
\end{proof}
\section{Helper Lemmas}
\label{sec:helper_lemmas}

\begin{lemma}[Factorization of Cartesian Products]
\label{lemma:cartesian-filter}
Let $\bs{R}_1, \bs{R}_2$ be two sets of random variables with disjoint supports,
and let $\phi_1, \phi_2$ be predicates depending only on configurations of
$\bs{R}_1$ and $\bs{R}_2$, respectively. Then
\begin{align*}
& \Big\{ (\bs{r}_1, \bs{r}_2) \in supp(\bs{R}_1) \times supp(\bs{R}_2)
  : \phi_1(\bs{r}_1) \wedge \phi_2(\bs{r}_2) \Big\}\\
&= \big\{ \bs{r}_1 \in supp(\bs{R}_1) : \phi_1(\bs{r}_1) \big\}
  \times \big\{ \bs{r}_2 \in supp(\bs{R}_2) : \phi_2(\bs{r}_2) \big\}.   
\end{align*}

\end{lemma}
\begin{proof}
$(\subseteq)$: If $(\bs{r}_1, \bs{r}_2)$ is in the left-hand side, then
$\phi_1(\bs{r}_1)$ holds, so $\bs{r}_1$ is in the first factor, and
$\phi_2(\bs{r}_2)$ holds, so $\bs{r}_2$ is in the second factor.

$(\supseteq)$: If $\bs{r}_1$ is in the first factor and $\bs{r}_2$ in the second, then both $\phi_1(\bs{r}_1)$ and $\phi_2(\bs{r}_2)$ hold, so $(\bs{r}_1, \bs{r}_2)$ is in the left-hand side.

The generalization to any finite collection $\bs{R}_1, \dots, \bs{R}_K$ follows by induction.
\end{proof}

\begin{lemma}[Summation over Cartesian Product as Double Summation]
\label{lemma:cartesian-double-sum}
Let $\mathbb{A}$ be one of $\mathbb{N}, \mathbb{Z}, \mathbb{Q}, \mathbb{R}, \mathbb{C}$. Let $S, T$ be finite sets and $f \colon S \times T \to \mathbb{A}$. Then
$$\sum_{(s,t) \,\in\, S \times T} f(s,t) = \sum_{s \in S}\; \sum_{t \in T} f(s,t).$$
\end{lemma}

\begin{proof}
We proceed by induction on $|T|$.

\paragraph{Base case.}
If $|T| = 0$, then $T = \emptyset$, so $S \times T = \emptyset$. Both sides equal $0$ by the convention for empty summation.

\paragraph{Inductive step.}
Suppose the identity holds whenever $|T| = n \geq 0$, and let $|T| = n + 1$. Pick $t_0 \in T$, set $T' := T \setminus \{t_0\}$. Then $S \times T = (S \times T') \cup (S \times \{t_0\})$ (disjoint). By the sum over disjoint unions:
$$\sum_{(s,t) \in S \times T} f(s,t) = \sum_{(s,t) \in S \times T'} f(s,t) + \sum_{(s,t) \in S \times \{t_0\}} f(s,t).$$
By the induction hypothesis, the first summand equals $\sum_{s \in S} \sum_{t \in T'} f(s,t)$. The second equals $\sum_{s \in S} f(s, t_0)$. Combining via linearity of summation:
$$= \sum_{s \in S} \left(\sum_{t \in T'} f(s,t) + f(s,t_0)\right) = \sum_{s \in S} \sum_{t \in T} f(s,t). \qedhere$$
\end{proof}

\section{Runtime and Scalability Experiments}
\label{app:runtime}

This appendix quantifies the computational profile of the pipeline. To state
the scope precisely: we do not claim a scalable solver for partial
identification. Potency evaluation relies on solving the underlying
response-type programs and necessarily inherits their computational cost. Our
contributions are instead (i) the formulation of experiment selection under
partial identification and its hardness (Theorem~\ref{thm:np-hard}), (ii) an
exact potency evaluation procedure wherever the underlying programs are
tractable, and (iii) sound graph-based certificates that reduce how many
expensive programs must be solved. Pruning only inspects the graph, so we
evaluate it on graphs with up to $76$ variables in
Section~\ref{sec:evaluation}. Measuring the solver time it \emph{saves}, by
contrast, means actually running the solver, so the experiments below are
limited to small instances---those on which the solver terminates at all,
using the single consumer laptop of Appendix~\ref{app:compute_used}. All
runtimes were measured on that machine (single-threaded), using the
\texttt{autobound} library and the potency procedure of
Appendix~\ref{subsec:nhanes-computation}. The experiments in this appendix
impose tighter per-solve time limits than the NHANES potencies---$60$ seconds
for the solvability check below, $600$ seconds for each candidate
evaluation---since a runtime measurement has to terminate, and here the limits
do bind.

\subsection{Scalability of Exact Potency Evaluation}
\label{subsec:scalability}

To probe how far the response-type programs themselves scale, we sampled
random instances (Erd\H{o}s--R\'enyi graphs, two-world counterfactual queries,
model-consistent random $P(\bs{V})$) and checked whether the observational
bounds alone are solvable within a $60$-second solver time limit: at
$|\bs{V}|=5$, $12$ of $25$ instances were solved (median $0.9$s); at
$|\bs{V}|=6$, $3$ of $25$ (median $5.1$s); at $|\bs{V}|=7$, $1$ of $25$
(solved in $1.3$s).\footnote{One instance at $|\bs{V}|=7$ is counted as
unsolved for a different reason: building its program exhausted the $4$~GB of
memory we allow per instance.}
Solving times appear strongly bimodal. One structural observation might explain this: by Lemma~\ref{th-all_constr_on_Rtheta_in_Dtheta_c-factor}, the size
of the optimization program is governed not by $|\bs{V}|$ but by the district
hull of the \emph{query}, so even a larger graph can lead to a small program
when the query has a localized district hull.\footnote{This is a structural statement
about when large graphs remain tractable, not a claim about how often
practical queries are localized.} In general, the max-potency problem and
polynomial programming remain NP-hard and thus computationally challenging.

\subsection{Runtime With and Without Pruning}
\label{subsec:runtime-pruning}

\paragraph{Selecting a single experiment.}
On a benchmark of $26$ randomly generated graphs with $|\bs{V}|=5$, drawn with
the same generation protocol as above at confounding densities
$c \in \{0.5, 1.0\}$ (sampled
until $26$ had observational bounds that solved within $60$ seconds, as above),
using all single-variable interventions supported by each graph ($2$--$12$
candidates per graph, $152$ in total), evaluating every candidate took at least
$26{,}753$ seconds ($7.4$ hours; each candidate was allowed a $600$-second
timeout), whereas evaluating only the survivors took $3{,}849$ seconds
($1.1$ hours)---a speedup of at least $6.95\times$. The pruning pass itself took
about $6$ milliseconds per instance.

The saving is larger than the drop in candidate count ($152$ to $48$) would
suggest, because the candidates the rules discard are also the slowest to
evaluate. Table~\ref{tab:runtime_by_class} shows the split. A candidate that
survives pruning is solved in $2.5$ seconds at the median, against $612$
seconds for a candidate discarded by the interception rule. Of the $50$
candidates discarded by that rule, $29$ did not finish within the
timeout; these contribute the time actually spent on them,\footnote{The
$600$~s limit is passed to the solver, which checks it only at node boundaries.
This is why the median of the interception group exceeds $600$~s.} so that
group's total and the $6.95\times$ are both lower bounds. Setting aside the
$39$ candidates that did not finish, the remaining $113$ take $2{,}240$ seconds
in total, against $204$ seconds for the $42$ of them that pruning keeps---an
$11\times$ saving.

We suspect the two groups differ because a discarded candidate is one whose
experiment cannot change the bounds: its constraint is never active, leaving the
objective flat across a large part of the search space, so the solver has little
to guide it and struggles to confirm optimality even though the answer is known
in advance to be zero. Pruning removes exactly these cases in our
Erd\H{o}s--R\'enyi graphs.

\begin{table}[H]
\centering\small
\caption{Time to evaluate a single candidate on the $|\bs{V}|=5$ benchmark,
grouped by what the pruning rules do with it. Candidates that did not finish
within the $600$-second timeout contribute their measured wall-clock time, so
the means and totals are under-estimates.}
\label{tab:runtime_by_class}
\begin{tabular}{lrrrrr}
\toprule
Candidate group & count & median & mean & unfinished & total \\
\midrule
Kept (must be evaluated)  & $48$  & $2.5$~s  & $80$~s  & $6$  & $3{,}849$~s \\
Discarded by the ID check & $54$  & $4.7$~s  & $49$~s  & $4$  & $2{,}659$~s \\
Discarded by interception & $50$  & $612$~s  & $405$~s & $29$ & $20{,}245$~s \\
\midrule
All candidates            & $152$ & $5.4$~s  & $176$~s & $39$ & $26{,}753$~s \\
\bottomrule
\end{tabular}
\end{table}

\paragraph{Subset search.}
On a randomly generated graph from the same generator ($|\bs{V}|=5$, $c=1.0$)
with $8$ experiments, unit costs,
and total budget $B=2$, evaluating the potency of all affordable sets required $105.7$ seconds
and $36$ programs; removing the candidates pruned by the ID criterion beforehand
reduced this to $62.7$ seconds and $21$ programs, while yielding exactly the same
optimum, as guaranteed by Corollary~\ref{thm-id_inertness}. Every program here
finished, so both figures are exact. This is an illustration of the resulting
saving on an instance small enough to enumerate exhaustively, not a scaling
study.

\paragraph{Baselines.}
Once the potencies have been computed, the optimal set is found by exhaustive
enumeration over the affordable subsets. One could instead use a cheaper
selection rule. On the same $8$-experiment instance with $B=2$, greedy selection
(repeatedly adding the experiment with the best potency gain per cost) matched
the exhaustive optimum exactly; picking the single most potent experiment that
fits the budget reached $53\%$ of it; choosing an affordable subset at random
reached $22\%$.\footnote{However, the
search strategy is not where the cost lives---greedy, random, and exhaustive
all rest on the same potency evaluations, and pruning cuts that shared cost
for all of them alike.} Greedy's strong showing is not a guarantee:
potency is not additive across experiments---in
Table~\ref{tab:nhanes_results}, the best singleton $\tilde{\mathcal{a}}_6$ is
not part of the optimal budget-$3$ set $\{\tilde{\mathcal{a}}_4,
\tilde{\mathcal{a}}_5\}$, so settling for the single best experiment reaches
only $64\%$ of the optimum there---and experiments that are useless alone can
matter in combination.

\subsection{Incompleteness of the Pruning Rules}
\label{subsec:pruning-incompleteness}

The pruning rules are one-sided certificates: they never discard a useful
experiment, but they might keep some useless ones. To quantify this, we
evaluated every surviving candidate numerically on the $|\bs{V}|=5$ benchmark
above ($26$ instances). In $10$ of them, pruning eliminated every candidate,
so no program had to be solved at all. The remaining $16$ contribute $102$
candidates in total, of which $48$ survived both rules. Six of these did not
finish within the
timeout; of the $42$ that did, $34$ turned out to have zero potency anyway
($81\%$).

\section{Experimental Setup Details}
\label{app:eval_setup}

This appendix details the synthetic evaluation protocol of Section~\ref{sec:evaluation}.

\subsection{Common Parameters}
\label{app:eval_common}

Both the Erd\H{o}s--R\'enyi and \emph{bnlearn} evaluations share the following configuration:
\begin{center}
\begin{tabular}{ll}
\toprule
$|\mathcal{A}|$ (candidate experiments per simulation) & $200$ \\
Master seed & $42$ \\
$\theta$ \texttt{CF\_fraction} & $1.0$ \\
$\theta$ outcome size $|\bs{Y}^{(\frak{c})}|$ & $1$ \\
$\theta$ intervention size $|\bs{X}^{(\frak{c})}|$ & $1$ \\
Experiment \texttt{CF\_fraction} & $0.3$ \\
Experiment outcome size $|\bs{W}^{(\frak{c})}|$ & $\mathrm{Unif}\{1,2,3\}$ \\
Experiment intervention size $|\bs{Z}^{(\frak{c})}|$ & $\mathrm{Unif}\{1,2,3\}$ \\
\bottomrule
\end{tabular}
\end{center}
We fix $\theta$ to be a two-world counterfactual to guarantee that the query is not point-identifiable, placing every simulation in the partial-identification regime our framework targets. The $30\%$ counterfactual share for candidate experiments reflects a plausible mix of interventional and counterfactual experiments a practitioner might consider.

\subsection{Query Sampling}
\label{app:eval_query}

Each query $\theta$ is sampled as follows:
\begin{enumerate}[leftmargin=*,topsep=2pt,itemsep=1pt]
    \item Pick $Y^{(1)}$ uniformly from $\bs{V}$; pick $X^{(1)}$ uniformly from $\mathrm{anc}(Y^{(1)}) \setminus \{Y^{(1)}\}$.
    \item For the second world, pick $Y^{(2)}$ uniformly from $\bs{V}$ subject to $X^{(1)} \in \mathrm{anc}(Y^{(2)})$, then set $X^{(2)} = X^{(1)}$.
\end{enumerate}

\subsection{Candidate Experiment Sampling}
\label{app:eval_experiments}

Each $\tilde{\mathcal{a}} \in \mathcal{A}$ is sampled independently. With probability $0.3$, $\tilde{\mathcal{a}}$ has two counterfactual worlds; otherwise one.
\begin{enumerate}[leftmargin=*,topsep=2pt,itemsep=1pt]
    \item Sample $|\bs{W}^{(1)}|, |\bs{Z}^{(1)}| \sim \mathrm{Unif}\{1,2,3\}$.
    \item Pick $\bs{W}^{(1)} \subseteq \bs{V}$ uniformly; pick $\bs{Z}^{(1)} \subseteq \mathrm{anc}(\bs{W}^{(1)}) \setminus \bs{W}^{(1)}$ uniformly.
    \item For two-world experiments, sample $\bs{W}^{(2)}, \bs{Z}^{(2)}$ analogously, but seed each with a shared element: pick $w \in \bs{W}^{(1)}$ uniformly and require $w \in \bs{W}^{(2)}$; pick $z \in \bs{Z}^{(1)} \cap (\mathrm{anc}(\bs{W}^{(2)}) \setminus \bs{W}^{(2)})$ uniformly and require $z \in \bs{Z}^{(2)}$. Remaining variables fill from $\bs{V}$ and $\mathrm{anc}(\bs{W}^{(2)}) \setminus \bs{W}^{(2)}$, respectively.
\end{enumerate}

\subsection{Erd\H{o}s--R\'enyi Graphs}
\label{app:eval_er}

Each ER graph has a fixed topological order $V_1 < \cdots < V_N$. Directed edges $V_i \to V_j$ for $i<j$ are drawn independently with probability $p$; bidirected edges between each pair are drawn independently with probability $q$.
\begin{center}
\begin{tabular}{ll}
\toprule
$N = |\bs{V}|$ & $\{10, 15, 20, 30\}$ \\
$p$ & $0.5 \cdot \log(N) / N$ \\
$q$ & $c/N$, with $c \in \{0.5, 1.0, 1.5, 2.0\}$ \\
Simulations per $(N, c)$ & $100$ \\
\bottomrule
\end{tabular}
\end{center}
The full sweep yields $4 \times 4 \times 100 = 1{,}600$ simulations, of which
$1{,}596$ are reported; in four the sampled graph admitted no query satisfying
Appendix~\ref{app:eval_query}, and the simulation was skipped.

\subsection{\emph{bnlearn} Graphs}
\label{app:eval_bnlearn}

We use $11$ networks from the \emph{bnlearn} repository: \texttt{ASIA}, \texttt{SACHS}, \texttt{ALARM}, \texttt{BARLEY}, \texttt{CHILD}, \texttt{INSURANCE}, \texttt{MILDEW}, \texttt{WATER}, \texttt{HAILFINDER}, \texttt{HEPAR2}, \texttt{WIN95PTS}. Only the graph structure is retained; conditional probability tables are discarded. Since these DAGs are not associated with latent variables, confounders are introduced synthetically. For each simulation, after sampling $\theta$ on the bare DAG:
\begin{enumerate}[leftmargin=*,topsep=2pt,itemsep=1pt]
    \item Add a forced confounder between $X^{(1)}$ and $Y^{(1)}$.
    \item Draw a target ratio $r \sim \mathrm{Unif}(0.1, 0.9)$, clamped below by $1/|\bs{V}|$ for achievability.
    \item Add confounders between uniformly random pairs of observed variables until the total reaches $\lfloor r \cdot |\bs{V}| \rfloor$.
\end{enumerate}
With $100$ simulations per graph, this yields $11 \times 100 = 1{,}100$ simulations.

\begin{figure}
    \centering
    \includegraphics[width=0.8\linewidth]{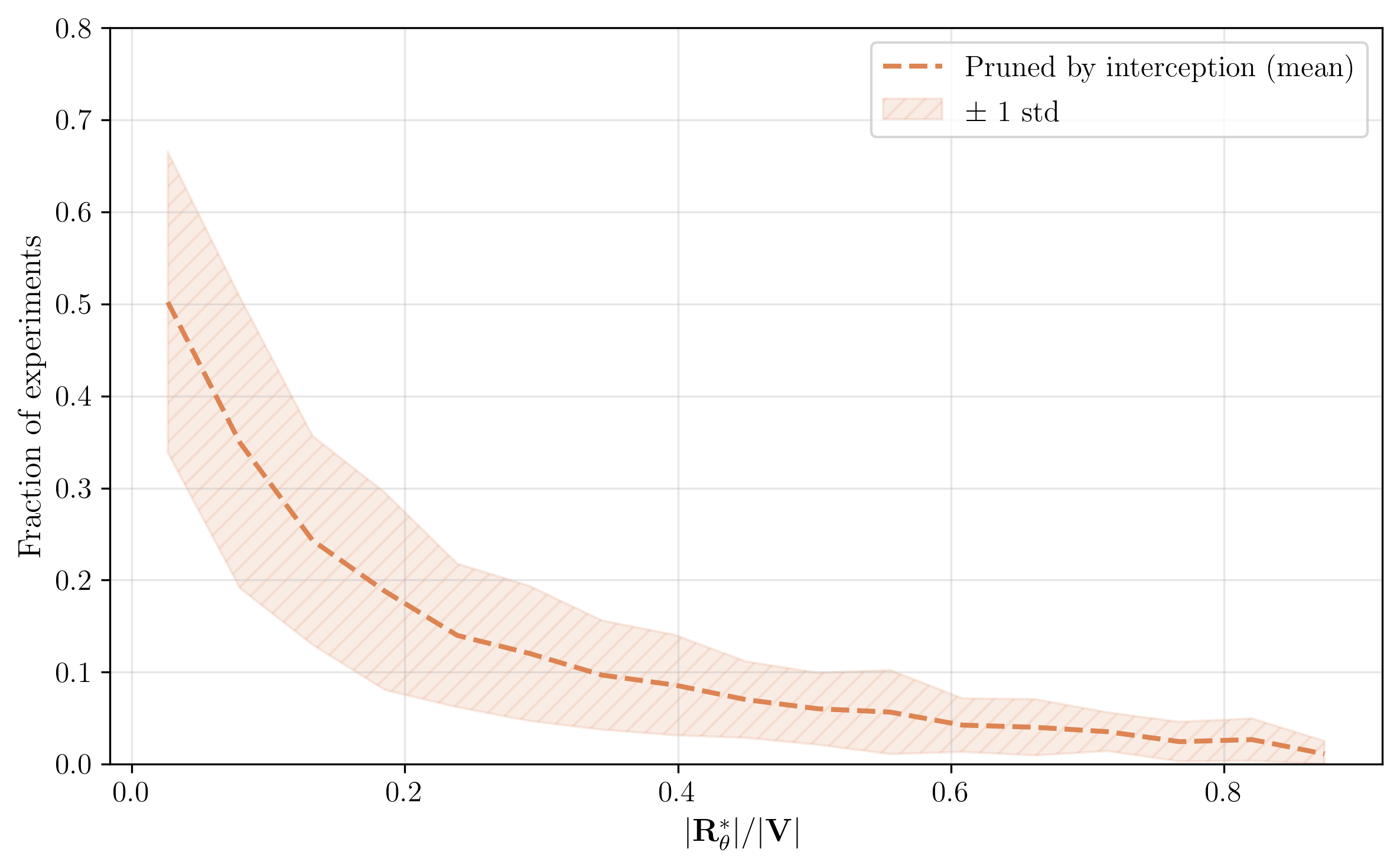}
    \caption{\emph{bnlearn}. The fraction of experiments pruned by the interception mechanism (Algorithm~\ref{alg:all_paths_intercepted}) against the relative size of $\bs{R}^*_\theta$.
    The line represents the mean fraction of experiments pruned; the band around it represents the $\pm 1$ standard deviation.}
    \label{fig:bnlearn_Rstartheta}
\end{figure}

\subsection{Compute Used}
\label{app:compute_used}
All experiments were executed on a single Microsoft Surface Laptop Studio with an 11th Gen Intel Core i7-11370H CPU (3.30~GHz base) and 32~GB RAM. No GPU was used and no parallelization was employed; all simulations ran single-threaded. Each of the three experimental components---the Erd\H{o}s--R\'enyi sweep (1{,}596 simulations), the \emph{bnlearn} sweep (1{,}100 simulations across 11 networks), and the NHANES end-to-end demonstration (Section~\ref{sec:evaluation}, Table~\ref{tab:nhanes_results})---completed in a matter of hours. No additional compute infrastructure was used.

\subsection{Per-graph results (\emph{bnlearn})}
\label{app:eval_bnlearn_table}
Table~\ref{tab:bnlearn_per_graph} reports the per-graph mean and standard
deviation of pruning rates across 100 simulations per graph, broken down by
mechanism. Across all 1{,}100 simulations, total pruning has mean $63.4\%$
(std $12.4\%$), of which $47.1\% \pm 14.5\%$ is contributed by the ID check
and $16.2\% \pm 17.5\%$ by interception. The interception rate varies by more
than an order of magnitude across networks---from $2.3\% \pm 4.3\%$ on
\texttt{MILDEW} to $29.2\% \pm 25.5\%$ on \texttt{WIN95PTS}---driven primarily
by the relative size of the district hull $|\bs{R}^*_\theta|/|\bs{V}|$
(Figure~\ref{fig:bnlearn_Rstartheta}).
\begin{table}[H]
\centering
\caption{Per-graph pruning rates on \emph{bnlearn} networks
(mean $\pm$ std over 100 simulations).}
\label{tab:bnlearn_per_graph}
\begin{tabular}{lccc}
\toprule
Graph & Interception & ID check & Total \\
\midrule
\texttt{ASIA}       & $0.170 \pm 0.134$ & $0.332 \pm 0.130$ & $0.502 \pm 0.158$ \\
\texttt{SACHS}      & $0.226 \pm 0.194$ & $0.316 \pm 0.113$ & $0.543 \pm 0.152$ \\
\texttt{INSURANCE}  & $0.073 \pm 0.073$ & $0.519 \pm 0.090$ & $0.592 \pm 0.113$ \\
\texttt{MILDEW}     & $0.023 \pm 0.043$ & $0.574 \pm 0.085$ & $0.597 \pm 0.092$ \\
\texttt{WATER}      & $0.138 \pm 0.136$ & $0.490 \pm 0.102$ & $0.628 \pm 0.092$ \\
\texttt{BARLEY}     & $0.122 \pm 0.122$ & $0.517 \pm 0.098$ & $0.639 \pm 0.083$ \\
\texttt{CHILD}      & $0.221 \pm 0.181$ & $0.428 \pm 0.124$ & $0.649 \pm 0.113$ \\
\texttt{ALARM}      & $0.173 \pm 0.173$ & $0.510 \pm 0.132$ & $0.682 \pm 0.070$ \\
\texttt{HEPAR2}     & $0.104 \pm 0.116$ & $0.579 \pm 0.095$ & $0.682 \pm 0.064$ \\
\texttt{HAILFINDER} & $0.243 \pm 0.193$ & $0.475 \pm 0.137$ & $0.718 \pm 0.073$ \\
\texttt{WIN95PTS}   & $0.292 \pm 0.255$ & $0.446 \pm 0.179$ & $0.738 \pm 0.085$ \\
\bottomrule
\end{tabular}
\end{table}

\subsection{Per-density breakdown (Erd\H{o}s--R\'enyi)}
\label{app:eval_er_table}

Table~\ref{tab:er_per_density} reports the breakdown of pruning rates across
the four confounding-density settings $c \in \{0.5, 1.0, 1.5, 2.0\}$, where
edges of the bidirected graph are drawn independently with probability $c/N$
(Appendix~\ref{app:eval_er}). The mean achieved confounder ratio $|\bs{U}|/|\bs{V}|$
grows monotonically with $c$, from $0.24$ at $c{=}0.5$ to $0.94$ at $c{=}2.0$.
As confounding density increases, interception weakens (mean drops from
$70.9\%$ to $16.5\%$) and the ID check picks up the slack (mean rises from
$17.2\%$ to $45.6\%$), but total pruning still declines from $88.1\%$ to
$62.1\%$.

\begin{table}[H]
\centering
\caption{Per-density breakdown for the Erd\H{o}s--R\'enyi evaluation
(mean $\pm$ std). $|\bs{U}|/|\bs{V}|$ is the mean confounder ratio for the given parameter $c$;
remaining columns are pruning rates by mechanism.}
\label{tab:er_per_density}
\begin{tabular}{lccccc}
\toprule
$c$ & $|\bs{U}|/|\bs{V}|$ & Interception & ID check & Total \\
\midrule
0.5 & 0.240 & $0.709 \pm 0.310$ & $0.172 \pm 0.190$ & $0.881 \pm 0.136$ \\
1.0 & 0.472 & $0.475 \pm 0.348$ & $0.304 \pm 0.216$ & $0.780 \pm 0.163$ \\
1.5 & 0.699 & $0.300 \pm 0.345$ & $0.389 \pm 0.218$ & $0.689 \pm 0.185$ \\
2.0 & 0.942 & $0.165 \pm 0.300$ & $0.456 \pm 0.201$ & $0.621 \pm 0.184$ \\
\bottomrule
\end{tabular}
\end{table}

\subsection{ID-only pruning (subset search)}
\label{app:eval_id_only}

For the subset search space reduction reported in
Section~\ref{sec:evaluation}, we run the ID check on all
$|\mathcal{A}|$ candidates independently, without prior interception
filtering. Tables~\ref{tab:bnlearn_id_only}
and~\ref{tab:er_id_only} report the results across the same
data-generating regime.

\begin{table}[H]
\centering
\caption{ID-only pruning on \emph{bnlearn} networks
(mean $\pm$ std over 100 simulations).}
\label{tab:bnlearn_id_only}
\begin{tabular}{lcc}
\toprule
Graph & $|\bs{V}|$ & ID-only \\
\midrule
\texttt{ASIA}       & 8  & $0.466 \pm 0.146$ \\
\texttt{SACHS}      & 11 & $0.479 \pm 0.121$ \\
\texttt{CHILD}      & 20 & $0.612 \pm 0.093$ \\
\texttt{INSURANCE}  & 27 & $0.582 \pm 0.093$ \\
\texttt{WATER}      & 32 & $0.600 \pm 0.074$ \\
\texttt{MILDEW}     & 35 & $0.611 \pm 0.088$ \\
\texttt{ALARM}      & 37 & $0.634 \pm 0.064$ \\
\texttt{BARLEY}     & 48 & $0.618 \pm 0.063$ \\
\texttt{HAILFINDER} & 56 & $0.660 \pm 0.050$ \\
\texttt{HEPAR2}     & 70 & $0.664 \pm 0.050$ \\
\texttt{WIN95PTS}   & 76 & $0.669 \pm 0.046$ \\
\midrule
Overall & & $0.600 \pm 0.108$ \\
\bottomrule
\end{tabular}
\end{table}

\begin{table}[H]
\centering
\caption{ID-only pruning for Erd\H{o}s--R\'enyi, pooled over
$|\bs{V}|$ (mean $\pm$ std).}
\label{tab:er_id_only}
\begin{tabular}{lccc}
\toprule
$c$ & $|\bs{U}|/|\bs{V}|$ & ID-only \\
\midrule
0.5 & 0.217 & $0.673 \pm 0.060$ \\
1.0 & 0.473 & $0.637 \pm 0.093$ \\
1.5 & 0.719 & $0.595 \pm 0.113$ \\
2.0 & 0.942 & $0.552 \pm 0.134$ \\
\midrule
Overall & & $0.614 \pm 0.113$ \\
\bottomrule
\end{tabular}
\end{table}

\section{NHANES Experiment Details}
\label{sec:nhanes_details}

This appendix details the construction of the NHANES experiment in
Section~\ref{sec:evaluation}. We use the NHANES 2017--2018
cycle\footnote{\url{https://wwwn.cdc.gov/nchs/nhanes/}} and restrict the
sample to adult respondents ($\text{RIDAGEYR}\ge 18$) with non-missing
entries on $Z$, $X$ and $Y$, yielding $n=5{,}833$ complete cases.

The assumed causal structure (Figure~\ref{fig:nhanes_dag} in
Section~\ref{sec:evaluation}) encodes the
following substantive assumptions: balance problems ($A$) may cause falls
($B$), and both may influence the propensity to engage in physical activity
($X$); health insurance ($Z$) acts as an instrument-like predictor of
physical activity through access to preventive care; and physical activity
affects the (negated) diabetes outcome ($Y$). Three latent confounders are
posited: $U_3$ between $A$ and $B$ (e.g., inner-ear or vestibular disorders
affecting both balance and falls), $U_1$ between $Z$ and $X$ (e.g., overall
healthcare access and socioeconomic factors), and $U_2$ between $X$ and $Y$
(e.g., genetic predisposition or general health-consciousness). The query
of interest is the average treatment effect $\theta = \mathrm{ATE} =
P(y_x) - P(y_{x'})$ of physical activity on not having diabetes.

Table~\ref{tab:nhanes_vars} maps each variable to its corresponding
NHANES variable code and the rule used to discretize it to a binary value.
We negate the diabetes indicator (denoting the variable $\neg Y$ in the
table) so that $Y=1$ corresponds to \emph{not} having a doctor-diagnosed
diabetes condition; this preserves the natural reading of $\theta>0$ as
``physical activity is protective.''

\begin{table}[H]
  \centering\small
  \caption{Mapping of variables to NHANES 2017--2018 codes. $A$ and $B$ are
  part of the assumed structure but are not recorded in this cycle.}
  \label{tab:nhanes_vars}
  \begin{tabular}{clllc}
    \toprule
    Node & Variable & Code & Discretization rule & $|\mathrm{supp}|$ \\
    \midrule
    $A$ & Balance problems & --- & not recorded in this cycle & 2 \\
    $B$ & Falls & --- & not recorded in this cycle & 2 \\
    $Z$ & Health insurance & HIQ011  & 1 if currently covered & 2 \\
    $X$ & Physical activity & PAQ650/665 & 1 if vigorous or moderate recreation & 2 \\
    $\neg Y$ & Diabetes & DIQ010  & 1 if doctor-diagnosed & 2 \\
    \bottomrule
  \end{tabular}
\end{table}

The five-variable ADMG is the input to the pruning step, which uses only the
graph and the query. The potency programs take the discretized observational
distribution $P(Z,X,Y)$.\footnote{$A$ and $B$ are not recorded in this cycle,
so they do not appear in $P(\bs{V})$. Since they influence the remaining
variables only through $X$ and share no latent parent with $Z$ or $Y$,
marginalizing them out leaves exactly the ADMG over $(Z,X,Y)$ used here, so
these are the bounds of the five-variable model.}

\subsection{Candidate Experiments: Realizable Trials vs.\ Assumption-Type Constraints}
\label{subsec:nhanes-experiments}

As defined in Section~\ref{sec:bounds}, the candidate set $\mathcal{A}$ includes
both physically realizable experiments and assumption-type constraints; the
polynomial program treats both uniformly, and the distinction is captured by the
cost function, which may assign infinite cost to experiments that cannot be
conducted (Section~\ref{sec:max_potency}). Whether a given counterfactual
quantity is physically realizable can be checked upfront with the algorithm of
\citet{raghavan2025counterfactual}. Table~\ref{tab:nhanes_experiment_types}
labels each candidate of Table~\ref{tab:nhanes_results} accordingly.

\begin{table}[H]
  \centering\small
  \caption{Candidates of Table~\ref{tab:nhanes_results}, labeled as realizable
  trials or assumption-type constraints.}
  \label{tab:nhanes_experiment_types}
  \begin{tabular}{clll}
    \toprule
     & Experiment & Type & Interpretation \\
    \midrule
    $\tilde{\mathcal{a}}_1$ & $P(X \mid do(B))$   & realizable & single-world intervention \\
    $\tilde{\mathcal{a}}_2$ & $P(B \mid do(A))$   & realizable & single-world intervention \\
    $\tilde{\mathcal{a}}_3$ & $P(X,Y \mid do(z))$ & realizable & single-world intervention \\
    $\tilde{\mathcal{a}}_4$ & $P(y_{x'},y'_{x})$  & assumption & cross-world counterfactual \\
    $\tilde{\mathcal{a}}_5$ & $P(Y \mid do(x'))$  & realizable & single-world intervention \\
    $\tilde{\mathcal{a}}_6$ & $P(Y \mid do(x))$   & realizable & single-world intervention \\
    \bottomrule
  \end{tabular}
\end{table}

Concretely, $\tilde{\mathcal{a}}_4 = P(y_{x'}, y'_{x})$ is the probability that
physical activity (necessarily and sufficiently) \emph{causes} diabetes. No
trial can measure it, but committing to $\tilde{\mathcal{a}}_4 = 0$ on
substantive grounds is exactly the monotonicity (``no harm'') assumption, which
is known to tighten PNS bounds \citep{tian_probabilities_2000}. Its cost of $1$
in Table~\ref{tab:nhanes_results} prices that commitment, not a trial. The
combinations $\{\tilde{\mathcal{a}}_4, \tilde{\mathcal{a}}_6\}$ and
$\{\tilde{\mathcal{a}}_4, \tilde{\mathcal{a}}_5\}$ mix both types.

\subsection{Computational Procedure}
\label{subsec:nhanes-computation}

Every potency in Table~\ref{tab:nhanes_results} is the optimum of the single
coupled polynomial program of Algorithm~\ref{alg:evaluate-potency}
(Appendix~\ref{subsec:evaluating-potency}); the experimental outcome is never
discretized. Programs are constructed following the procedure of
\citet{duarte2023automated}, using the implementation in the supplementary
material, and solved by Couenne at its default tolerances---no custom optimality gap, and a $300$-second time limit that no
solve came close to---i.e.\ to global optimality up to solver tolerance; early
termination plays no role in any reported number. The computation is reproducible from
\texttt{nhanes\_pipeline.ipynb} in the supplementary material.

\subsection{Robustness}
\label{subsec:nhanes-robustness}

\paragraph{Finite-sample stability.}
The pipeline uses the plug-in principle and treats the empirical distribution
of the $n = 5{,}833$ NHANES records as exact. Methodologically, sampling error
can be accounted for within the same machinery:
\citet[Section~6]{duarte2023automated} relax each observational equality
constraint to a $1{-}\alpha$ confidence region, and since our experimental
constraints and the potency program are built on the same constraint set, the
relaxation carries over unchanged; the pruning step uses only the graph and
the query, never $P(\bs{V})$, so pruning decisions are unaffected by sampling
error by construction. To quantify what the plug-in treatment neglects, we
bootstrapped the sample ($50$ replicates) and recomputed the potencies. The
observational width averaged $0.86 \pm 0.03$ across replicates (range $0.82$ to
$0.94$), and all potencies had standard deviations below $0.025$ (e.g.,
$\tilde{\mathcal{a}}_3$: $0.000 \pm 0.000$; $\tilde{\mathcal{a}}_4$: $0.438 \pm
0.009$; $\{\tilde{\mathcal{a}}_4, \tilde{\mathcal{a}}_5\}$: $0.774 \pm 0.024$).
The optimal choice at $B = 1, 2,
3$ agreed with Table~\ref{tab:nhanes_results} in all $50$ replicates. The
potency programs use only $P(Z,X,Y)$, whose $8$ cells are estimated from
$5{,}833$ records, so this stability is expected.

\paragraph{Graph misspecification.}
Dependence on the assumed graph is fundamental to any method in this space, but
our framework separates two kinds of dependence. Pruning depends on the graph
only through the district hull $\bs{R}_\theta^*$ of the query. The potency
values depend on more than that: by
Lemma~\ref{th-all_constr_on_Rtheta_in_Dtheta_c-factor} they are determined by
the hull together with the mechanisms of the variables inside it, so an edit
that changes one of those mechanisms can change the potencies even when the
hull itself is untouched.

Adding a direct edge $Z \to Y$---dropping the assumption that health insurance
affects diabetes only through physical activity---separates the two. The hull
$\bs{R}_\theta^*$ is unchanged, and so is every pruning decision in
Table~\ref{tab:nhanes_results}: $\tilde{\mathcal{a}}_1$ is still certified by
the ID check and $\tilde{\mathcal{a}}_2$ still by the interception rule. But
the edge makes $Z$ a parent of $Y$, and $Y$ lies inside the hull, so the
existing potencies change. The observational width rises from $0.86$ to $1.00$,
because the bounds can no longer use the assumption that $Z$ acts on $Y$ only
through $X$, and every potency rises with it: $\tilde{\mathcal{a}}_4$ from
$0.44$ to $0.48$, $\tilde{\mathcal{a}}_6$ from $0.50$ to $0.53$, and
$\{\tilde{\mathcal{a}}_4, \tilde{\mathcal{a}}_5\}$ from $0.78$ to $0.85$. The
optimal choice at each budget $B = 1, 2, 3$ is nevertheless the same under both
graphs.